\documentclass[letterpaper,12pt]{extarticle}
\usepackage[margin=1in,letterpaper]{geometry}
\usepackage[utf8]{inputenc}
\usepackage{amsmath,amssymb,amsthm}
\usepackage{nicefrac,enumerate}
\usepackage{algorithm}
\usepackage{algpseudocode}
\usepackage{hyperref}
\usepackage{lipsum}
\usepackage{amsfonts}
\usepackage{graphicx}
\usepackage{epstopdf}
\usepackage{subcaption}

\newtheorem{theorem}{Theorem}[section]
\newtheorem*{maintheorem}{Theorem}
\newtheorem{corollary}[theorem]{Corollary}
\newtheorem{definition}[theorem]{Definition}
\newtheorem{proposition}[theorem]{Proposition}
\newtheorem{lemma}[theorem]{Lemma}
\newtheorem{ex}[theorem]{Example}
\newtheorem*{remark}{Remark}

\newcommand{\relu}{\text{ReLU}}
\newcommand{\I}{I_x^{\alpha}}
\newcommand{\Iy}{I_y^{\alpha}}
\newcommand{\Io}{I_{x_0}^{\alpha}}
\newcommand{\spn}{\text{span}}
\newcommand{\fphi}{\langle x,\phi_i\rangle}
\newcommand{\yphi}{\langle y,\phi_i\rangle}

\newcommand{\Ta}{C_\alpha}
\newcommand{\RR}{\mathbb{R}}
\newcommand{\inv}{^{-1}}

\usepackage{tcolorbox}
\tcbset{width = \textwidth-2cm, center,before skip=20pt plus 2pt,after skip=20pt plus 2pt}

\title{Injectivity of ReLU layers:\\ Tools from Frame Theory}

\author{
Daniel Haider\footnote{Acoustics Research Institute, Vienna, Austria (daniel.haider@oeaw.ac.at).},
Martin Ehler\footnote{University of Vienna, Faculty of Mathematics, Vienna, Austria (martin.ehler@univie.ac.at).}, and
Peter Balazs\footnote{Acoustics Research Institute, Vienna, Austria (peter.balazs@oeaw.ac.at).}
}
\date{November 2024}

\begin{document}

\maketitle

\begin{abstract}
Injectivity is the defining property of a mapping that ensures no information is lost and any input can be perfectly reconstructed from its output. By performing hard thresholding, the ReLU function naturally interferes with this property, making the injectivity analysis of ReLU layers in neural networks a challenging yet intriguing task that has not yet been fully solved.
This article establishes a frame theoretic perspective to approach this problem.
The main objective is to develop a comprehensive characterization of the injectivity behavior of ReLU layers in terms of all three involved ingredients: $(i)$ the weights, $(ii)$ the bias, and $(iii)$ the domain where the data is drawn from. Maintaining a focus on practical applications, we limit our attention to bounded domains and present two methods for numerically approximating a maximal bias for given weights and data domains. These methods provide sufficient conditions for the injectivity of a ReLU layer on those domains and yield a novel practical methodology for studying the information loss in ReLU layers. Finally, we derive explicit reconstruction formulas based on the duality concept from frame theory.


\end{abstract}

\section{Introduction}
The \textbf{Re}ctified \textbf{L}inear \textbf{U}nit defined as $\relu(s)=\max(0,t)$ for $t\in \RR$ has become indispensable as a non-linear activation function in artificial neural networks \cite{deepsparse11,hin12,dlb,relu10}. Since originally introduced as a way to regularize the gradients in deep network architectures,
there have been hardly any networks that do not use ReLU activation or some derivation of it \cite{clevert2016elu, he2015prelu, maas2013lrelu}.

A \emph{ReLU layer} $C_\alpha(x) = \relu(Cx-\alpha)$ is the composition of an affine linear map comprising the multiplication by a weight matrix $C\in \RR^{m\times n}$ and the shift by a bias vector $\alpha\in \RR^m$, with an entry-wise application of ReLU on its output. The injectivity of a ReLU layer, and with that, the possibility of inverting it and inferring $x$ from $C_\alpha(x)$, is a desired property in various applications. There has been interest in building injective models and inverting them on their range to regularize ill-posed inverse problems or designing injective generative models, such as normalizing flows, for manifold learning or compressed sensing \cite{kothari2021injflows,puthawala2022injflows,bora2017compressed}. Generally, knowing if a layer involves a loss of information increases the interpretability of the network immensely. As such, one can use an inverse mapping to trace back each layer output to its source input, which can help to decipher the decision-making process, diagnose model behavior, identify biases, and
study accountability. Although the injectivity of ReLU layers has received increasing attention in recent years, it is still not fully understood, especially when it comes to the numerical verification in practice.\\

The goal of this paper is to demystify the injectivity of a single ReLU layer as a deterministic non-linear map on a comprehensive level. This involves a thorough analysis of fundamental properties of ReLU layers that are relevant for applications, and different characterizations of injectivity with respect to all properties involved, namely weights, bias, and input domain. By translating selected theoretical results into algorithmic solutions we present novel ways of verifying injectivity on bounded input domains in practice. Explicit reconstruction formulas and their implementation, together with a brief local stability analysis complete the claim of the paper.\\

The methodology to achieve these goals is based on \textit{frame theory}, a mathematical paradigm that deals with stable, potentially redundant, and invertible representations of functions by means of inner products \cite{frames}. To make use of this machinery and all tools that come with it, we shall consider a weight matrix $C\in \RR^{m\times n}$ in terms of its row vectors
\begin{equation}\label{eq:init}
    C = \begin{pmatrix}
    -  \phi_1  -\\
      \vdots  \\
    -  \phi_m  -
\end{pmatrix}.
\end{equation}
If $C$ has more rows than columns ($m\geq n$) and full rank, the collection of row vectors $(\phi_i)_{i=1}^m$ is a spanning set for the domain space $\RR^n$. In other words, the associated linear transform $C: \RR^n \rightarrow \RR^m$ is injective. In the context of frame theory, we say that $(\phi_i)_{i=1}^m$ is a \emph{frame} for $\RR^n$ \cite{frames} and $C$ is the associated \textit{analysis operator}. The application of $C$ to $x$ is interpreted as measuring the correlation of
$x$ to all frame vectors $\phi_i$ via $x \mapsto (\fphi)_{i=1}^m$.
The resulting so-called \textit{frame coefficients} $(\fphi)_{i=1}^m$ give a (potentially redundant) representation of $x$, from which we can always infer $x$ explicitly. Roughly speaking, frame theory is the study of ``quantifying'' injectivity of a redundant representation in the sense of its numerical stability and constructing recovery maps with desired properties via the concept of dual frames. In this sense, it provides exactly the right tools for the goal of the paper.

Although frame theory deals with linear representations, it has been shown to be suitable for non-linear problems as well. One example is phase-retrieval \cite{balan06} where, in the real setting, one asks for the injectivity and stability of the map
\begin{align}\label{eq:phase}
\begin{split}
   C_{\vert\;.\;\vert}: \mathbb{R}^n/\{\pm 1\}&\rightarrow \mathbb{R}^m\\
   x&\mapsto \left(\big\vert \langle x,\phi_i\rangle\big\vert \right)_{i=1}^m.
\end{split}
\end{align}
It is well known that $C_{\vert\;.\;\vert}$ is injective if and only if for any partition of the collection $(\fphi)_{i=1}^m$ into two sub-collections, at least one of them is a frame \cite{balan06}.
Inspired by this approach, we may write a ReLU layer analogously as the map
\begin{align}\label{eq:reluintro}
\begin{split}
    C_\alpha:\RR^n &\rightarrow \mathbb{R}^m\\
    x&\mapsto\left(\relu(\langle x,\phi_i\rangle-\alpha_i)\right)_{i=1}^m.
\end{split}
\end{align}
In \cite{puth22} (using another terminology) it has been shown that $C_\alpha$ is injective if and only if for any $x\in \RR^n$ the frame vectors that are not affected by ReLU are a frame. We shall call a frame with this property $\alpha$\emph{-rectifying} on $\RR^n$. While this characterization forms the basis of our work, we will focus on the practical assumption that, in applications, it may not always be most informative to consider the entire $\mathbb{R}^n$ as the domain where a ReLU layer should be injective. In fact, when considering standard normalization schemes of data sets for training and testing neural networks, it seems more reasonable to study injectivity only on bounded subsets $K\subseteq \mathbb{R}^n$ where the data is assumed or processed to live in. Indeed, a ReLU layer might be injective on $K$ but not on $\RR^n$. In this paper, we show that the choice of the data domain can have a profound impact on the injectivity behavior of $C_\alpha$, and discuss how to leverage this fact in practice. Prototypical examples of such domains include the closed ball in $\RR^n$ of radius $r>0$, given by $\mathbb{B}_r=\{x\in\RR^n:\|x\|\leq r\},$ the closed donut arising by excluding small data points $\mathbb{D}_{r,s} = \overline{\mathbb{B}_r\setminus\mathbb{B}_s}$ with $s<r$, and the sphere $\mathbb{S}=\{x\in\RR^n:\|x\|= 1\}$ \cite{backprop12, normalization20}. Furthermore, when considering two consecutive ReLU layers, we can restrict the injectivity property of the second one to domains that lie in $\RR^n_{+}$, such as the non-negative closed ball $\mathbb{B}_r^+ = \mathbb{B}_r\cap\RR^n_{+}$.
The restriction of the domain of $\Ta$ from $\RR^n$ to a bounded $K\subseteq \mathbb{R}^n$ increases the feasibility and applicability of the problem in practice, while also making the mathematical setting more versatile. Furthermore, it establishes a natural connection to the bias vector and provides a framework where we can control the injectivity behavior through these two ingredients. This in turn allows us to approach the injectivity analysis also algorithmically.\\

The main theoretical component of this paper is a comprehensive characterization of the injectivity of a ReLU layer as a deterministic map, summarized in the following theorem.

\begin{maintheorem}\label{thm:bigthm}
    Let $\Phi=(\phi_i)_{i=1}^m$ be a frame for $\RR^n$, $\alpha \in \mathbb{R}^m$, and $\emptyset\neq K\subseteq \RR^n$. Under the assumptions that $\Phi$ includes a unique most correlated basis everywhere (Def \ref{def:mostcorr}), $K$ is open or strictly convex, and bias-exact for $\Phi$ (Def. \ref{def:uniquemaxa}), the following are equivalent.
    \begin{enumerate}
        \item[(i)] The ReLU layer $\Ta$, associated with $\Phi$ and $\alpha$, is injective on $K$.
        \item[(ii)] The frame $\Phi$ is $\alpha$-rectifying on $K$
        (Def \ref{alpharect}, Thm \ref{reluinj1}, \ref{reluinj0}).
        \item[(iii)] The domain $K$ lies in the maximal domain $\mathcal{K}_{\alpha}^*$ (Thm \ref{thm:bas}).
        \item[(iv)] The values of the bias $\alpha$ do not exceed the values of the maximal bias $\alpha_K^{\sharp}$ (Thm \ref{thm:maxa}).
    \end{enumerate}
    For any bias $\alpha$ the maximal domain $\mathcal{K}_{\alpha}^*$ can be constructed explicitly as the union of intersections of closed affine half-spaces. For any domain $K$ the maximal bias $\alpha_K^{\sharp}$ can be approximated numerically via sampling or via the inscribing polytope associated with $\Phi$.
\end{maintheorem}
The main practical component of the paper comprises two algorithmic constructions of biases that approximate the maximal bias $\alpha_K^{\sharp}$ from the theorem above in different situations. These can be used to study and effectively control the injectivity behavior of a ReLU layer in practice.
Moreover, using the duality concept from frame theory we derive inversion formulas for injective ReLU layers that can be implemented easily as locally linear operators.

\subsubsection*{Related work}
The approach in classical phase-retrieval in $\RR^n$ by Balan et al. in \cite{balan06} was decisive for the idea of characterizing the injectivity of a ReLU layer in terms of a property of the associated frame. The same approach is taken by Alharbi et al. to study the recovery of vectors from saturated inner measurements \cite{alharbi2024sat}.
In a machine learning context, Puthawala et al. have introduced the notion of \emph{directed spanning sets} in \cite{puth22} as an equivalent concept to the \emph{admissibility} condition by Bruna et al. in \cite{bruna14} to characterize a ReLU layer to be injective on $\RR^n$. While the primary goal in \cite{puth22} is the study of globally injective ReLU-networks on $\RR^n$, and the one in \cite{bruna14} is a Lipschitz stability analysis of ReLU layers, our goal is to demystify the injectivity of a single ReLU layer in a more realistic setting, namely with any given weights and biases on bounded input \textit{data} domains $K\subseteq\RR^n$, and to provide methods of verifying injectivity in practice. This extends the ideas and methods by Haider et al. introduced in \cite{haider2023relu} significantly.
Further related preprints are by Behrmann et al., who study the pre-images of ReLU layers from a geometric point of view \cite{behr18} and by Maillard et al., which focus on injectivity of ReLU layers with random weights \cite{maillard2023injectivity}.



\subsubsection*{Outline}
The paper is divided into four sections.
In Section \ref{sec:basic}, we introduce \textit{$\alpha$-rectifying frames} as a characterizing family of frames that are associated with $\relu$ layers that are injective on a given input domain and discuss fundamental properties that are crucial for applications. In Section \ref{sec:interplay}, we study the interplay of input domain and bias vector and derive a maximal domain and a maximal bias such that the associated ReLU layers become critically injective. This leads to two further characterizations of injectivity.
Moreover, we present two methods to approximate the maximal bias and provide algorithmic solutions to apply them in practice. Section \ref{chap:duality} is dedicated to explicit reconstruction formulas for injective ReLU layers and a brief local stability analysis of the recovery map.

\section{Frames and the Injectivity of ReLU Layers}\label{sec:basic}
When applying a matrix $C\in \RR^{m\times n}$ to a vector $x\in \RR^n$ we can reconstruct $x$ if and only if the collection of row vectors of $C$ spans $\RR^n$. Frame theory offers a definition of a spanning set that allows us to quantify ``how good'' it spans $\RR^n$.
So, throughout this paper, we denote a collection of $m$ vectors by $\Phi=(\phi_i)_{i\in I}\subset \mathbb{R}^n$, using index sets $I$ with $\vert I\vert=m\geq n$. Then, $\Phi$ is a \textit{frame} for $\RR^n$ if there are constants $0< A\leq B$, called the \textit{frame bounds} of $\Phi$, such that
\begin{equation}\label{eq:framedef}
    A\cdot \|x\|^2\leq \sum_{i\in I}\vert \fphi \vert^2 \leq B\cdot\| x\|^2
\end{equation}
for all $x\in\mathbb{R}^n$ \cite{casfin12}. For $J\subseteq I$, we denote by $\Phi_J=(\phi_i)_{i\in J}$ the sub-collection of $\Phi$ with respect to the index set $J$. Any frame with $m=n$ is a basis and we will always assume that $\Phi$ does not contain the zero vector. It is easy to see that \eqref{eq:framedef} is equivalent to $\Phi$ being a spanning set and that the frame bounds $A, B$ reflect the numerical stability properties of the representation of $x$ under $\Phi$.
The \emph{analysis operator} associated with $\Phi$ is given as
\begin{align}\label{eq:frameana}
\begin{split}
    C:\mathbb{R}^n&\rightarrow \mathbb{R}^m\\
    x&\mapsto\left(\fphi\right)_{i\in I}
\end{split}
\end{align}
mapping a vector $x$ to its \textit{frame coefficients} as discussed in the introduction. So finally, we have that $\Phi$ is a frame, if and only if $C$ is injective.
Together with the assumption on the input domain mentioned in the introduction, this motivates to define a ReLU layer as follows.
\begin{definition}[ReLU layer]\label{relulayer}
A $\relu$ layer associated with a collection of weight vectors $\Phi= (\phi_i)_{i\in I}\subset \RR^n$, a bias vector $\alpha = (\alpha_1,...,\alpha_m)^{\top}\in \mathbb{R}^m$ and an input domain $K\subseteq \RR^n$ is defined as the non-linear map given by
\begin{align*}
    C_\alpha:K &\rightarrow \mathbb{R}^m\\
    x&\mapsto\left(\operatorname{ReLU}(\langle x,\phi_i\rangle-\alpha_i)\right)_{i=1}^m.
\end{align*}
\end{definition}
To encode the injectivity of $C_\alpha$ directly in terms of $\Phi$ we introduce a family of frames called \textit{$\alpha$-rectifying frames.}

\subsection{Alpha-rectifying frames}
For any given $x\in K$ the shift by the bias and the application of the ReLU function on the frame coefficients act as a thresholding mechanism that neglects all frame elements $\phi_i$ where $\fphi < \alpha_i$, rendering them \textit{inactive}. According to this observation, for $x\in \mathbb{R}^n$ and $\alpha\in\mathbb{R}^m$ we are interested in the index set associated with those frame elements that are \textit{active} for $x$ and $\alpha$. We shall denote it by
\begin{equation}\label{eq:index}
    \I=\{i\in I:\langle x,\phi_i\rangle\geq\alpha_i\}.
\end{equation}
Dual to this notion, for $i\in I$ and $\alpha\in \RR^m$ we denote the closed affine half-space of points where the frame element $\phi_i$ is active by
\begin{equation}\label{eq:omega0}
    \Omega_i^\alpha=\{x\in \mathbb{R}^n:\langle x,\phi_i\rangle\geq\alpha_i\}.
\end{equation}
The following definition gives the frame theoretic perspective to \cite[Definition 1]{puth22}.
\begin{definition}[$\alpha$-rectifying frames]\label{alpharect}
The collection $\Phi=(\phi_i)_{i\in I}\subset \mathbb{R}^n$ is called $\alpha$-rectifying on $K\subseteq \RR^n$ for $\alpha \in \mathbb{R}^m$ if for all $x \in K$ the sub-collection of active frame elements $\Phi_{\I} = (\phi_i)_{i\in \I}$ is a frame for $\mathbb{R}^n$. 
\end{definition}
Figure \ref{fig:fig1} illustrates the two notions in \eqref{eq:index} and \eqref{eq:omega0} on the closed unit ball $\mathbb{B}$ in $\RR^2$ (left, mid), and shows a simple example of an $\alpha$-rectifying frame (right).

While many papers only consider ReLU layers without bias vectors and with input on whole $\RR^n$~\cite{puth22, maillard2023injectivity}, for our approach, bias vectors and the input domain play important roles, and the following basic properties of ReLU layers become crucial. Throughout the paper, for $\alpha,\alpha'\in \RR^m$ we shall use the notation $\alpha\leq \alpha'$ for $\alpha_i\leq \alpha'_i$ for all $i\in I$, and $\alpha< \alpha'$ for $\alpha_i< \alpha'_i$ for all $i\in I$.
\begin{proposition}\label{prop:inclusive}
    Let $\Phi$ be $\alpha'$-rectifying on $K'$. The following holds.
    \begin{enumerate}[(i)]
        \item $\Phi$ is $\alpha'$-rectifying on $K$ for every $K\subseteq K'$.
        \item $\Phi$ is $\alpha$-rectifying on $K'$ for every $\alpha\leq \alpha'$.
    \end{enumerate}
\end{proposition}
Hence, we are naturally interested in knowing the largest possible domains and biases that allow the $\alpha$-rectifying property. In Section \ref{sec:interplay} we prove that indeed one obtains full characterizations of the injectivity of a ReLU layer via a \textit{maximal domain} and a \textit{maximal bias}.

\begin{figure}[t]
    \centering
    \begin{subfigure}[t]{0.30\textwidth}
        \centering
        \includegraphics[width=\textwidth]{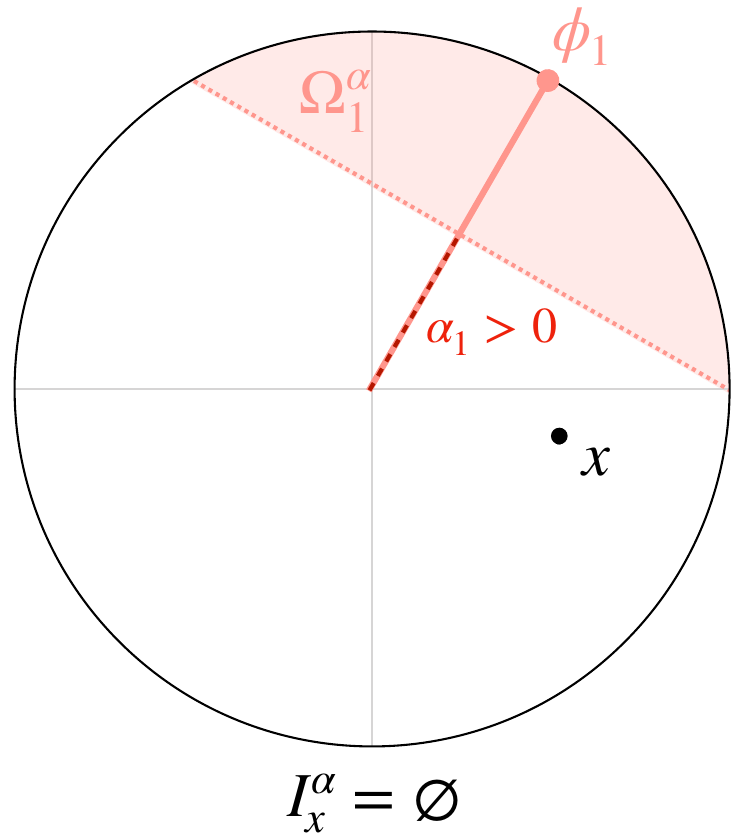}
        
    \end{subfigure}%
    \hfill
    \begin{subfigure}[t]{0.30\textwidth}
        \centering
        \includegraphics[width=\textwidth]{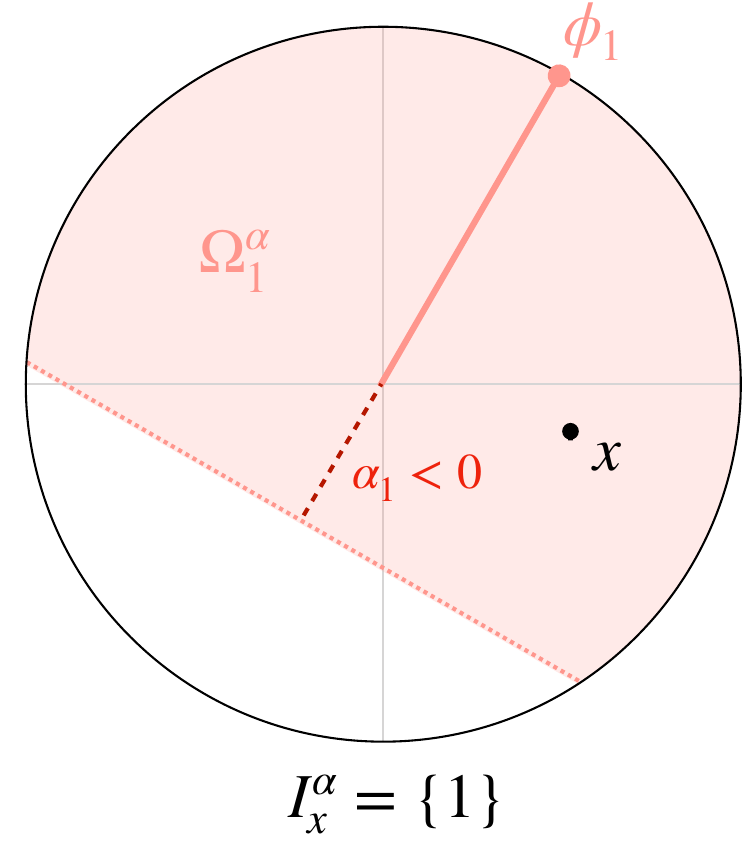}
    \end{subfigure}%
    \hfill
    \begin{subfigure}[t]{0.32\textwidth}
        \centering
        \includegraphics[width=\textwidth]{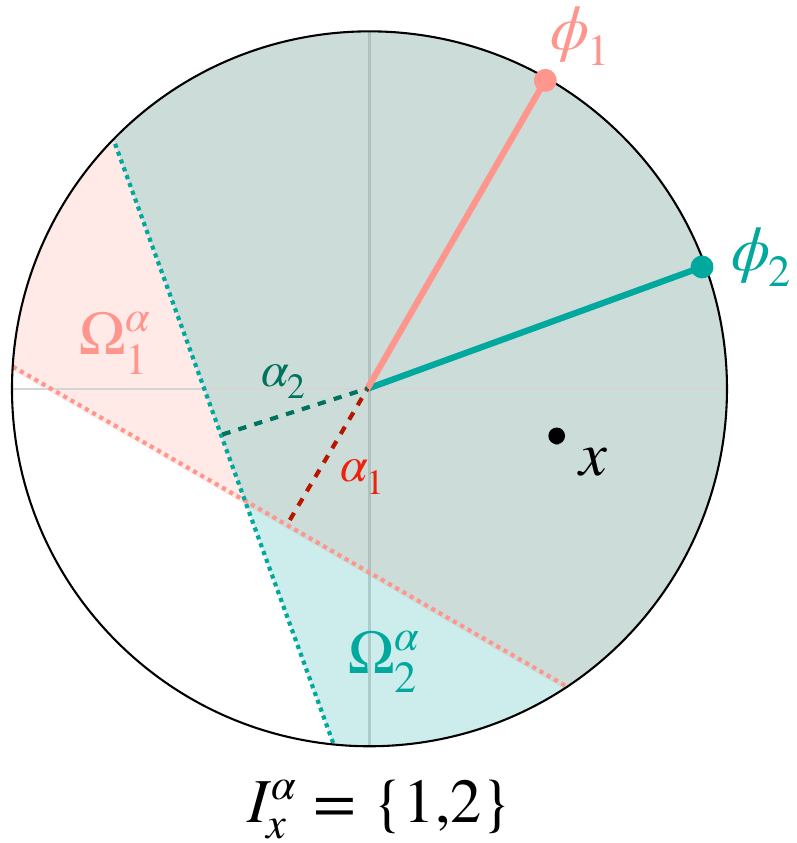}
    \end{subfigure}%

    \caption{Illustrations of the notions $\I$ and $\Omega_i^\alpha$ related to active frame elements on $\mathbb{B}$. The frame $(\phi_1,\phi_2)$ in the most right example is $\alpha$-rectifying on $K$ if $K\subseteq (\Omega_1^\alpha\cap \Omega_2^\alpha)$.}
    \label{fig:fig1}
\end{figure}

We now present the fundamental connection between the $\alpha$-rectifying property and the injectivity of $\Ta$ on $K$, forming the backbone of this paper. Theorems \ref{reluinj1} and \ref{reluinj0} address the two directions separately, each focusing on specific properties of $K$, thereby generalizing \cite[Theorem 2]{puth22}.

\begin{theorem}[Injectivity of ReLU layers I]\label{reluinj1}
Given $\Phi=(\phi_i)_{i\in I}\subset \mathbb{R}^n$, $\alpha \in \mathbb{R}^m$, and $\emptyset\neq K\subseteq \RR^n$. If $\Phi$ is $\alpha$-rectifying on $K$, then $C_\beta$ is injective on $K$ for all $\beta<\alpha$.
Moreover, if $K$ is open or convex, $C_\alpha$ is injective on $K$.
\end{theorem}

Recall that a set $U$ is called strictly convex if for all $x,y\in U$ and $\lambda \in (0,1)$, $x_\lambda:= (1-\lambda)x - \lambda y \in \mathring{U} $, where $\mathring{U}$ is the interior of $U$. In particular, $\mathring{U}\neq \emptyset$.

\begin{theorem}[Injectivity of ReLU layers II]\label{reluinj0}
Given $\Phi=(\phi_i)_{i\in I}\subset \mathbb{R}^n$, $\alpha \in \mathbb{R}^m$, and let $\emptyset\neq K\subseteq \RR^n$ be open or strictly convex. If $C_\alpha$ is injective on $K$, then $\Phi$ is $\alpha$-rectifying on $K$.
\end{theorem}

\begin{proof}[Proof of Theorem \ref{reluinj1}]
Let $x,y\in K$ and assume $C_\beta x=C_\beta y$. Clearly,
\begin{equation}\label{eq:000}
\fphi > \beta_i\; \text{ if and only if }\; \yphi > \beta_i,
\end{equation}
from which we can deduce that $\langle x,\phi_i\rangle = \langle y,\phi_i\rangle$ for all $i\in I^\beta_x$.
By assumption that $\Phi$ is $\alpha$-rectifying, $\Phi_{I^\alpha_x}$ is a frame, and since $\beta < \alpha$ we have that $\Phi_{I^\beta_x}$ is one, too. It follows that $x=y$ and therefore that $C_\beta$ is injective on $K$.

To show the moreover part, let $C_\alpha x=C_\alpha y$, for $x,y\in K$. Clearly,
\begin{equation}
\langle x,\phi_i\rangle = \langle y,\phi_i\rangle 
\end{equation}
for all $i \in \I\cap\Iy$. We will show that if $K$ is open or convex, then $\Phi_{\I\cap \Iy}$ is a frame.\\
Let us first consider the case where $K$ is open. We may choose $\varepsilon > 0$ such that the open ball around $x$, denoted by $B_\varepsilon ^\circ (x)$, is contained in $K$. When assuming that $x\neq y$ then there is $\delta<1$ with $0 < \delta < \varepsilon \cdot \|x-y\|^{-1}$ such that
\begin{equation}\label{eq:00e}
    x_\delta := (1-\delta)x + \delta y \in B_\varepsilon ^\circ (x).
\end{equation}
Now let $i\in I_{x_\delta}^\alpha$. By the linearity of the inner product, we have the following.
\begin{align}
    \text{If }\;\langle x_\delta , \phi_i \rangle &> \alpha_i\; \text{ then }\; \fphi > \alpha_i \text{ and } \yphi > \alpha_i.\label{eq:del1}\\
    \text{If }\;\langle x_\delta , \phi_i \rangle &= \alpha_i\; \text{ then }\; \fphi = \alpha_i \text{ and } \yphi = \alpha_i.\label{eq:del2}
\end{align}
Therefore, $I_{x_\delta}\subseteq \I\cap \Iy$. By \eqref{eq:00e}, we have that $x_\delta\in K$ and since we assumed $\Phi_{I_{x_\delta}}$ to be a frame, so is $\Phi_{\I\cap \Iy}$.\\
Now let us assume $K$ to be convex. For $\lambda \in (0,1)$,
\begin{equation}\label{eq:lam}
    x_\lambda := (1-\lambda)x + \lambda y \in K.
\end{equation}
By the same arguments as above, \eqref{eq:del1} and \eqref{eq:del2} hold for $\langle x_\lambda , \phi_i \rangle$, hence $I_{x_\lambda}\subseteq \I\cap \Iy$. By assumption, $\Phi_{I_{x_\lambda}}$ is a frame and thereby, $\Phi_{\I\cap \Iy}$ is a frame. So for both cases, we can deduce that $x=y$, hence $\Ta$ is injective.
\end{proof}

\begin{proof}[Proof of Theorem \ref{reluinj0}]
We prove the claim by counterposition. Assume that $\Phi$ is not $\alpha$-rectifying. Then there is $x\in K$ such that $(\phi_i)_{i\in \I}$ is not a frame. Hence, there is 
\begin{equation}\label{eq:123}
0\neq r\in\spn(\phi_i)_{i\in \I}^{\bot}.
\end{equation}
If $K$ is open, for all sufficiently small $\varepsilon>0$ we have that
$$
y_\pm:=x\pm \varepsilon r\in K.
$$
For $i\in\I$, \eqref{eq:123}
implies that
$
\langle y_+,\phi_i\rangle = \langle y_-,\phi_i\rangle,
$
leading to 
\begin{equation}\label{eq:007}
\max(0,\langle y_+,\phi_i\rangle-\alpha_i)=\langle x,\phi_i\rangle -\alpha_i= \max(0,\langle y_-,\phi_i\rangle-\alpha_i).
\end{equation}
If $\I=I$, then \eqref{eq:007} already implies $C_\alpha y_+=C_\alpha y_-$, so that $C_\alpha$ is not injective on $K$.\\
To address the case $\I\subset I$, we recall that 
$
\langle x,\phi_i\rangle <\alpha_i
$
holds for all $i\in I\setminus \I$.
Therefore, we may choose $\varepsilon$ sufficiently small, such that $y_\pm \in K$ and
\begin{equation}\label{eq:008}
\langle y_\pm,\phi_i\rangle <\alpha_i,\qquad i\in I\setminus \I.
\end{equation}
Observation \eqref{eq:008} leads to
\begin{equation}\label{eq:009}
0=\max(0,\langle y_\pm,\phi_i\rangle-\alpha_i),\qquad i\in I\setminus\I.
\end{equation}
According to \eqref{eq:007} and \eqref{eq:009}, we derive $C_\alpha y_+=C_\alpha y_-$. Thus, $C_\alpha$ is not injective on $K$.

Note that if $K$ is strictly convex it contains more than one element and for every $z \in K$ where $z\neq x$ and $\lambda \in (0,1)$ with the same assumptions on $x$ as above,
$$
x_\lambda:= (1-\lambda)x - \lambda z \in \mathring{K}.
$$
Using the linearity of the inner product, there is $\lambda$ sufficiently small such that
\begin{equation}\label{eq:124}
    \langle x_\lambda,\phi_i \rangle < \alpha_i, \qquad i\in I\setminus \I.
\end{equation}
Since $\mathring{K}$ is open and not empty, we can apply the same argument as above with $x_\lambda$ to show that $C_\alpha$ is not injective on $K$.
\end{proof}
This shows that the injectivity of a ReLU layer is contingent upon topological properties of the domain from which the data is drawn.
For the cases $\mathbb{R}^n$ and $\mathbb{B}_r$, we have established equivalence, where the former case corresponds to \cite[Theorem 2]{puth22}. Since the non-negative ball $\mathbb{B}_r^+ = \mathbb{B}_r\cap \RR^n_+$ is convex but not open or strictly convex, only the direction in Theorem \ref{reluinj1} holds. In the case of the donut $\mathbb{D}_{r,s} = \overline{\mathbb{B}_r\setminus \mathbb{B}_s}$, which is not open and not convex, only the first part of Theorem \ref{reluinj1} holds, i.e., injectivity for strictly smaller biases. In Section \ref{sec:maxK}, we give an example of a similar scenario where injectivity fails.\\

While the $\alpha$-rectifying property makes the injectivity of a ReLU layer more accessible, it remains challenging to verify it in practice. In the following, we discuss how specific properties of the frame influence its $\alpha$-rectifying property, and how to leverage them.

\subsection{Normalized frames}\label{sec:norm}
It is a simple, yet, crucial observation that we may restrict the $\alpha$-rectifying property to frames with unit norm vectors by scaling the bias with the norms of the frame elements. In the context of neural networks, normalizing the underlying frame (or the weight matrix in a row-wise manner) is a standard normalization technique \cite{weightnorm16}. We write $\Phi \subset \mathbb{S}$.
\begin{lemma}
    A frame $\Phi=(\phi_i)_{i\in I}\subset \RR^n$ is $\alpha$-rectifying on $K$ if and only if the normalized frame $\Phi'=\left(\phi_i \cdot \|\phi_i\|^{-1}\right)_{i\in I}\subset \mathbb{S}$  is $\alpha'$-rectifying on $K$, where $\alpha'_i=\alpha_i \cdot \|\phi_i \|^{-1}$.
\end{lemma}
The statement follows from the fact that $\langle x,\phi_i \rangle \geq \alpha_i$ is equivalent to $\langle x,\phi_i \cdot \|\phi_i\|^{-1} \rangle \geq \alpha_i \cdot \|\phi_i\|^{-1}$ for all $x\in K$. 
Therefore, we may always assume $\|\phi_i\|=1$ for all $i\in I$, which simplifies the problem setting substantially. The norms can be reintroduced at any stage of processing or analysis.

\subsection{Full-spark frames}\label{sec:fullspark}
The effect of ReLU can be interpreted as introducing input-dependent erasures in the underlying frame, i.e., losing certain coefficients \cite{goyal2001erasures}. In this context, the \textit{spark} of a frame has been shown to be a useful concept. It is defined as the smallest number $s \geq n+1$ of linearly dependent frame elements that one can choose from $\Phi$. In other words, any sub-collection with $s-1$ frame elements from $\Phi$ is a frame.
Frames with $s=n+1$ are called \textit{full-spark frames}. This family of frames has shown to be maximally robust to erasures \cite{alexeev2012spark} which makes the full-spark property interesting for injective ReLU layers in particular.
For phase retrieval, it is known that if a full-spark frame has $m\geq 2n-1$ elements then the phase-retrieval operator \eqref{eq:phase} is injective \cite{balan06}.
For the ReLU case, knowing the spark of $\Phi$ relaxes the condition for a frame to be $\alpha$-rectifying to a counting argument.
\begin{corollary}
    Let $\Phi$ be a frame with spark $s$ then $\Phi$ is $\alpha$-rectifying on $K$ if and only if $\vert \I \vert \geq s-1$ for all $x\in K$. 
\end{corollary}
Although it is an NP-hard problem to verify if a given frame is full-spark, in a numerical setting it is a mild condition that is almost surely satisfied in the presence of randomness. Indeed, if the entries of the frame elements are i.i.d.~samples from an absolutely continuous probability distribution, then the associated random frame is full-spark with probability one \cite{alexeev2012spark}.
Since most initialization methods in neural networks are based on i.i.d.~sampling schemes, full-spark frames appear naturally in the context of deep learning.

\subsection{Perturbed frames}
For a bounded domain $K$ the $\alpha$-rectifying property is robust to perturbation. In particular, small perturbations of an $\alpha$-rectifying frame result in an $\alpha'$-rectifying frame where $\alpha'$ is close to $\alpha$.
\begin{lemma}\label{lem:perturbation}
    Let $\Phi$ be $\alpha$-rectifying on a bounded domain $K$ with $M = \sup_{x\in K}\Vert x\Vert$. For $\varepsilon>0$, a perturbed frame $\Phi'=(\phi'_i)_{i\in I}$ satisfying $\|\phi_i-\phi'_i\|<\varepsilon$ for all $i\in I$ is $\alpha'$-rectifying on $K$ with $\alpha'_i =\alpha_i - \varepsilon M,\; i\in I$.
\end{lemma}
\begin{proof}
    Let $x\in K$, then for any $i\in \I$ it holds that
    \begin{align*}
        \langle x, \phi'_i \rangle = \langle x, \phi_i \rangle - \langle x, \phi_i - \phi'_i \rangle > \alpha_i-\varepsilon \Vert x\Vert \geq \alpha_i- \varepsilon M.
    \end{align*}
\end{proof}
It is known that for any frame $\Phi$ there is an arbitrarily small perturbation such that the resulting perturbed frame is full-spark \cite{casfin12}.
Therefore, we may interpret Lemma \ref{lem:perturbation} in the sense that for any $\alpha$-rectifying frame, there is an arbitrarily close full-spark frame that is $\alpha'$-rectifying with $\alpha'$ being arbitrarily close to $\alpha$.

\subsection{Redundancy}\label{sec:red}
One of the central properties of a frame is its redundancy, i.e., the ratio $q = \frac{m}{n}\geq 1$. For random ReLU layers with i.i.d.~Gaussian entries and no bias, it has been studied at which redundancy they become injective on $\RR^n$ asymptotically \cite{puth22,maillard2023injectivity}.
Let $p_{m,n}$ denote the probability that $C_\alpha$ with $m$ frame elements in $\RR^n$ and $\alpha = \mathbf{0}$ is injective then it has been proven that
$q\leq 3.3$ implies that $\lim_{n\rightarrow \infty}p_{m,n}=0$ and $q\geq 9.091$ implies that $\lim_{n\rightarrow \infty}p_{m,n}=1.$
Furthermore, the authors in \cite{maillard2023injectivity} state the conjecture that there exists a redundancy $q\in (6.6979, 6.6981)$ where the transition from non-injectivity to injectivity happens. We will revisit this conjecture in a non-asymptotic setting in the experimental part later in the paper (Section \ref{sec:experiments}).

In a non-random setting, a trivial leverage of redundancy is considering the collection $\Psi = \left(\Phi , -\Phi\right)$ for any given frame $\Phi$. Doubling the redundancy in this symmetric way makes $\Psi$ become $\mathbf{0}$-rectifying on $\RR^n$ by construction (see Figure \ref{fig:fig2} left) \cite{bruna14, puth22, framelets18}. In a general deterministic setting, however, it is difficult to establish sufficient conditions for the $\alpha$-rectifying property only in terms of redundancy as it depends heavily on the geometric characteristics of the frame. As already mentioned in Section \ref{sec:fullspark}, this is in contrast to the phase-retrieval setting, where a redundancy of $q\geq \frac{2n-1}{n}$, together with a full spark assumption is already sufficient for injectivity.
In the ReLU case, it is known that a redundancy of two is necessary when considering $\alpha = \mathbf{0}$ and $K=\RR^n$ \cite{bruna14}. We extend this known result from $\alpha = \mathbf{0}$ to arbitrary $\alpha$ in the following proposition.
\begin{proposition}
    Any $\alpha$-rectifying frame on $\RR^n$ has at least redundancy two.
\end{proposition}\label{prop:red}
\begin{proof}
By assuming that the zero vector is not a frame element, we can choose $x\in \RR^n$ with $\langle x,\phi_i \rangle\neq 0$ for all $i\in I$. We denote
\begin{align}
    I_{x}^+ &= \{i\in I:\fphi >0\},\\
    I_{x}^- &= \{i\in I:\fphi <0\}.
\end{align}
By the choice of $x$, the sets $I_{x}^+$ and $I_{x}^-$ form a disjoint partition of $I$.
Let $\alpha\in \RR^m$ and define
\begin{equation}
    t^* := \max_{\substack{i\in I}} \frac{\alpha_i}{\fphi}
\end{equation}
then for all $t>t^*>0$ we have
\begin{align}
    &\langle tx,\phi_i \rangle > \alpha_i\; \text{ and }\; \langle -tx,\phi_i \rangle < \alpha_i \quad \text{for } i\in I_x^+\\
    &\langle tx,\phi_i \rangle < \alpha_i\; \text{ and }\; \langle -tx,\phi_i \rangle > \alpha_i \quad \text{for } i\in I_x^-.
\end{align}
Hence, $I_{tx}^{\alpha}=I_x^+$ and $I_{-tx}^{\alpha}=I_x^-$.
We found two elements $u=tx,v= -tx\in \RR^n$ with $I_{u}^\alpha\cap I_{v}^\alpha = \emptyset$. Assuming $\Phi$ to be $\alpha$-rectifying on $\RR^n$ implies that $\Phi_{I_{u}^\alpha}$ and $\Phi_{I_{v}^\alpha}$ are frames, i.e., contain at least $n$ elements in particular. Since $I_{u}^\alpha\cap I_{v}^\alpha = \emptyset$, it must hold that $m\geq 2n$.
\end{proof}
Assuming that the input for a ReLU layer is contained in $\mathbb{B}_r$, we find that the necessary redundancy-two condition from Proposition \ref{prop:red} breaks. We use boldface notation for bias vectors with constant entries, i.e., $\mathbf{r}\in \RR^m$ denotes the vector with entries $r\in \RR$.
\begin{lemma}\label{lem:Br}
    Any normalized frame is $(-\mathbf{r})$-rectifying on $\mathbb{B}_r$.
    If it is additionally a basis, this is also necessary, i.e., $-r$ is the maximal value.     
\end{lemma}\label{lem:red}
\begin{proof}
    Let $x\in \mathbb{B}_r$. Since $\fphi = \Vert x \Vert \langle \frac{x}{\Vert x \Vert} , \phi_i\rangle \geq - \Vert x \Vert \geq -r$ the first statement follows. For the second, let $\Phi$ be a basis, then $\Phi$ is $\alpha$-rectifying on $\mathbb{B}_r$ if and only if for every $x\in \mathbb{B}_r$ it holds that $I=\I$. In particular, since $-r\cdot \Phi\subset \mathbb{B}_r$ and for every $i\in I$, the maximal choice of the $\alpha_i$ is determined by the fact that $\langle -r\cdot \phi_i,\phi_i \rangle = -r$.
\end{proof}
This emphasizes that the choice of the input domain can have a significant impact on the $\alpha$-rectifying property. The following section will examine this interaction between domain and bias in greater detail and explain how it can be leveraged.

\section{Interplay of Domain and Bias}\label{sec:interplay}
In the context of applications, we may find ourselves in a situation where we have provided a trained ReLU layer and wish to ascertain whether it is injective for a specific data set. One way to address this is to verify that the data set in question is contained within a set where we have already established that the ReLU layer is injective. This leads to the following natural question.
\begin{center}
    \textit{\textbf{Q1:} Given $\alpha$, what is the largest domain $K$ such that $\Phi$ is $\alpha$-rectifying on $K$?}
\end{center}
An alternative approach, building upon the inclusiveness property of ReLU layers (Prop. \ref{prop:inclusive}), is to ascertain that the values of the given bias do not exceed the values of a bias for which we already know that the corresponding ReLU layer is injective. This leads to the dual question to the one above.
\begin{center}
    \textit{\textbf{Q2:} Given $K$, what is the largest bias $\alpha$ such that $\Phi$ is $\alpha$-rectifying on $K$?}
\end{center}
Answering these questions will provide us with two further characterizations of the $\alpha$-rectifying property in terms of domain and bias, respectively. With this, we obtain alternative ways of verifying the injectivity of the associated ReLU layer.
To better understand how bias and domain interact, we point out some basic scaling relations.

\begin{lemma}\label{prop:K}
    Let $\Phi$ be $\alpha$-rectifying on $K$. The following holds.
    \begin{enumerate}[(i)]
        \item $\Phi$ is $(r\cdot \alpha)$-rectifying on $r\cdot K$ for any $r>0$.
        \item If $\alpha\geq 0$, then $\Phi$ is $\alpha$-rectifying on $r\cdot K$ with $r\geq 1$.
        \item If $0\in K$, then at least $n$ bias values are non-positive.
    \end{enumerate}
\end{lemma}

\begin{proof}
    All properties are easy to see.
    \begin{enumerate}[(i)]
        \item For $r>0$, $\langle x,\phi_i \rangle\geq \alpha_i $ if and only if $ \langle r\cdot x,\phi_i \rangle\geq r\cdot\alpha_i$.
        \item For $r\geq 1$, $\langle x,\phi_i \rangle\geq \alpha_i $ implies $ \langle r\cdot x,\phi_i \rangle\geq r\cdot\alpha_i\geq \alpha_i$.
        \item For $x=0$, we have that $\Phi_{\I}$ is a frame and $0 = \langle 0,\phi_i \rangle \geq \alpha_i$ holds for all $i\in I_x^\alpha$. Since $|\I|\geq n$, the claim follows.
        \end{enumerate}
\end{proof}
Consequently, by either scaling the data or the bias, it may be possible to compensate for situations where a frame is not $\alpha$-rectifying on $K$ but is on $K'\subsetneq K$.
The following examples show such compensation through restrictions other than scaling.

\begin{figure}[t]
    \centering
    \begin{subfigure}[t]{0.31\textwidth}
        \centering
        \includegraphics[width=\textwidth]{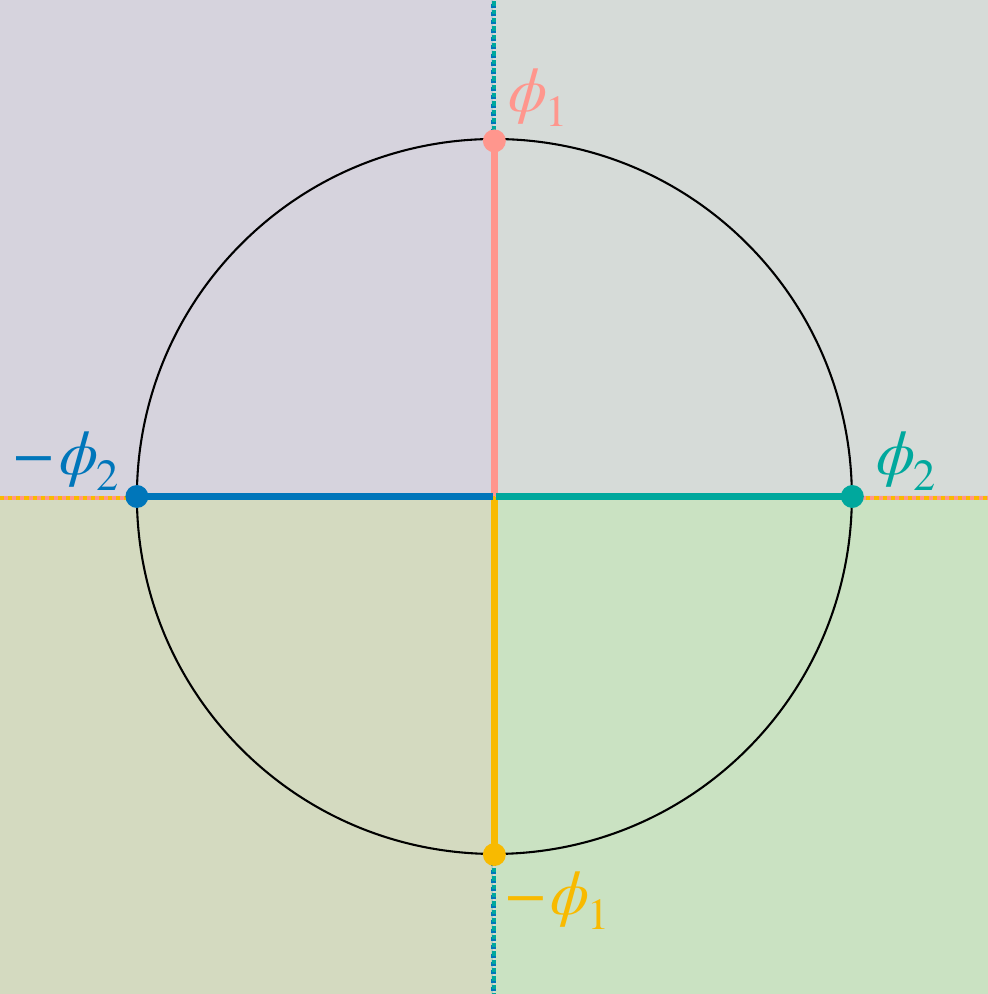}
        
    \end{subfigure}%
    \hfill
    \begin{subfigure}[t]{0.31\textwidth}
        \centering
        \includegraphics[width=\textwidth]{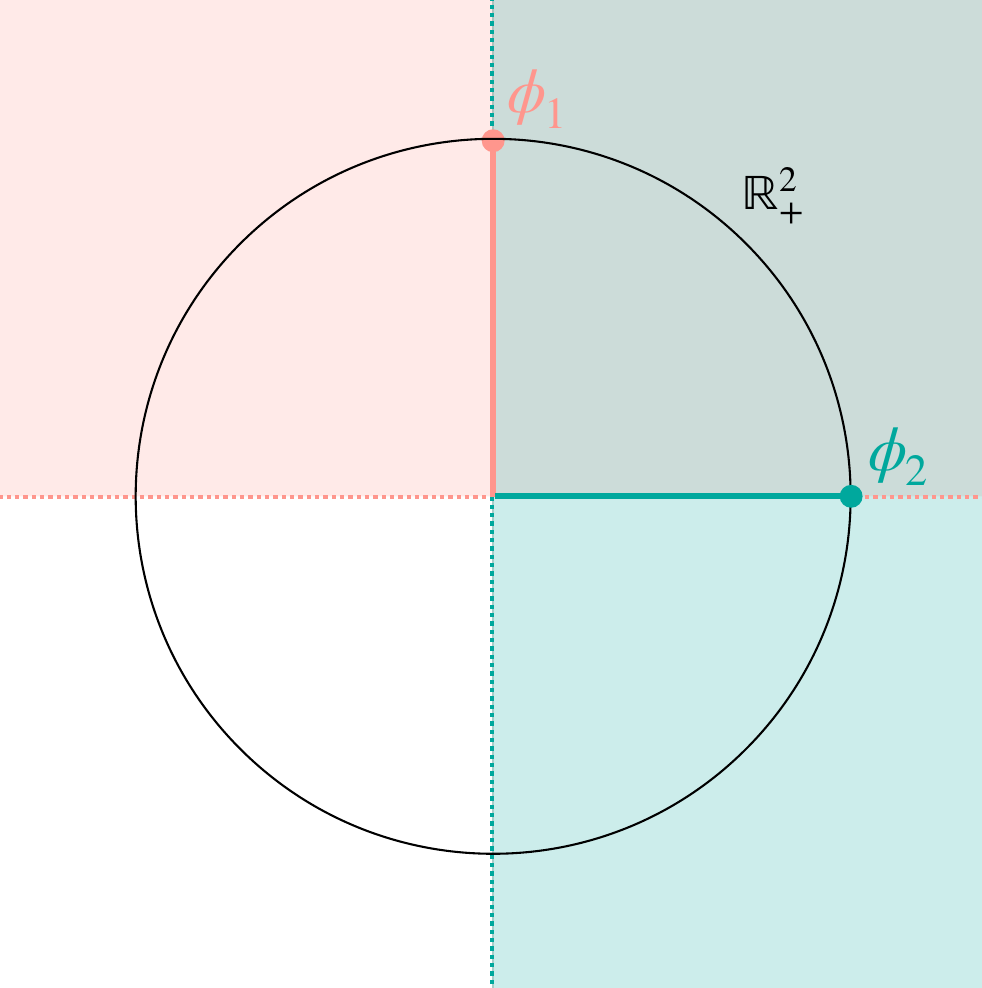}
    \end{subfigure}%
    \hfill
    \begin{subfigure}[t]{0.31\textwidth}
        \centering
        \includegraphics[width=\textwidth]{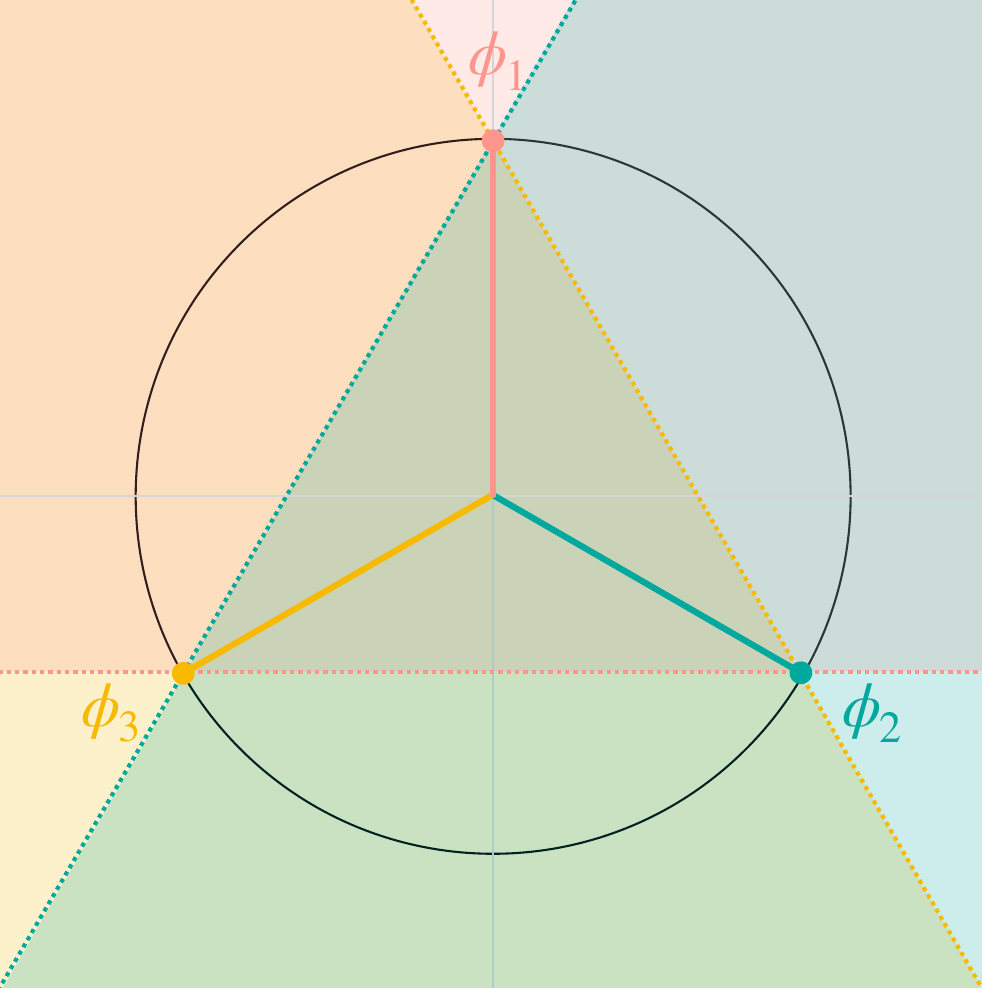}
    \end{subfigure}%

    \caption{Left: The frame composed of the standard basis and its negative elements is $\mathbf{0}$-rectifying on $\RR^2$. Mid: The standard basis is $\mathbf{0}$-rectifying on $\RR^2_+$ and $(-\mathbf{1})$-rectifying on $\mathbb{B}$. Right: The triangle frame is $(-\mathbf{\frac{1}{2}})$-rectifying on $\mathbb{B}$, but never on $\RR^2$ since there will always be cones where only one element is active (lighter areas).}
    \label{fig:fig2}
\end{figure}

\begin{ex}
    A basis can never be $\alpha$-rectifying on $\RR^n$ for any $\alpha$. However, the standard basis for $\RR^n$ is $\mathbf{0}$-rectifying on $\RR^n_+$ (Figure \ref{fig:fig2} mid).
\end{ex}

\begin{ex}\label{ex:mb}
    The frame
        $$
        \Phi_3=\left(
        \begin{pmatrix}
        0 \\
        1 
        \end{pmatrix},
        \begin{pmatrix}
        -\nicefrac{\sqrt{3}}{2} \\
        -\nicefrac{1}{2}
        \end{pmatrix},
        \begin{pmatrix}
        \nicefrac{\sqrt{3}}{2} \\
        -\nicefrac{1}{2}
        \end{pmatrix}
        \right)
        $$
    is not $\alpha$-rectifying on $\RR^n$ for any $\alpha$ since $m=3<4=2n$. However, by a geometric argument (see Figure \ref{fig:fig2} right), it is easy to see that $\Phi$ is $(-\mathbf{\frac{1}{2}})$-rectifying on $\mathbb{B}$. Note that $-\tfrac{1}{2}$ is the largest possible value here. 
\end{ex}

This makes clear that it is essential to select the domain carefully if we want to effectively study the injectivity behavior of the associated ReLU layer. This leads us to Question \textbf{\textit{Q1}}.

\subsection{Maximal domain}\label{sec:maxK}
We aim to identify the maximal domain $K$ for a frame $\Phi$ and a bias $\alpha$. This provides a characterization of the $\alpha$-rectifying property from a geometric point of view.
Recall that for $i\in I$ and $\alpha\in \RR^m$ we denote the closed affine half-space where the frame element $\phi_i$ is active for $\alpha$ by
\begin{equation}\label{eq:omega}
    \Omega_i^\alpha=\{x\in \mathbb{R}^n:\langle x,\phi_i\rangle\geq\alpha_i\}.
\end{equation}
Extending the intuition from the example in Figure \ref{fig:fig1} (right), we find that any frame is $\alpha$-rectifying on the intersection of sufficiently many $\Omega_i^\alpha$'s. Restricting to minimal frames, i.e., basis, reveals the characterization.

\begin{theorem}[Maximal domain]\label{thm:bas}
Let $\Phi\subset \mathbb{R}^n$ be a frame and $\alpha \in \mathbb{R}^m$. The maximal domain where $\Phi$ is $\alpha$-rectifying is given by
\begin{equation}\label{eq:bas}
    \mathcal{K}_{\alpha}^*=\bigcup_{\substack{J\subseteq I\\ \Phi_J\ \text{basis}}} \bigcap_{i\in J}\Omega_i^\alpha.
\end{equation}
In other words, $\Phi$ is $\alpha$-rectifying on $K$ if and only if $K\subseteq \mathcal{K}_{\alpha}^*$.
\end{theorem}

\begin{proof}
Let $x\in K$. Assuming $\Phi$ to be $\alpha$-rectifying on $K$, then $\Phi_{\I}$ is a frame. Since (in $\RR^n$) every frame contains a basis, there is $L\subseteq \I$ such that the sub-collection $\Phi_L$ is a basis. Clearly, $\fphi\geq\alpha_i$ still holds for all $i\in L$, hence, $x\in \mathcal{K}_{\alpha}^*$.

For the converse direction, by definition of $\mathcal{K}_{\alpha}^*$ for all $x\in \mathcal{K}_{\alpha}^*$ there is $M\subseteq I$ with $\fphi\geq \alpha_i$ for all $i\in M$ such that $\Phi_M$ is a basis. Since $M\subseteq \I$, it follows that $\Phi_{\I}$ is a frame.
\end{proof}

\begin{figure}[t]
    \centering
    \begin{subfigure}[t]{0.31\textwidth}
        \centering
        \includegraphics[width=\textwidth]{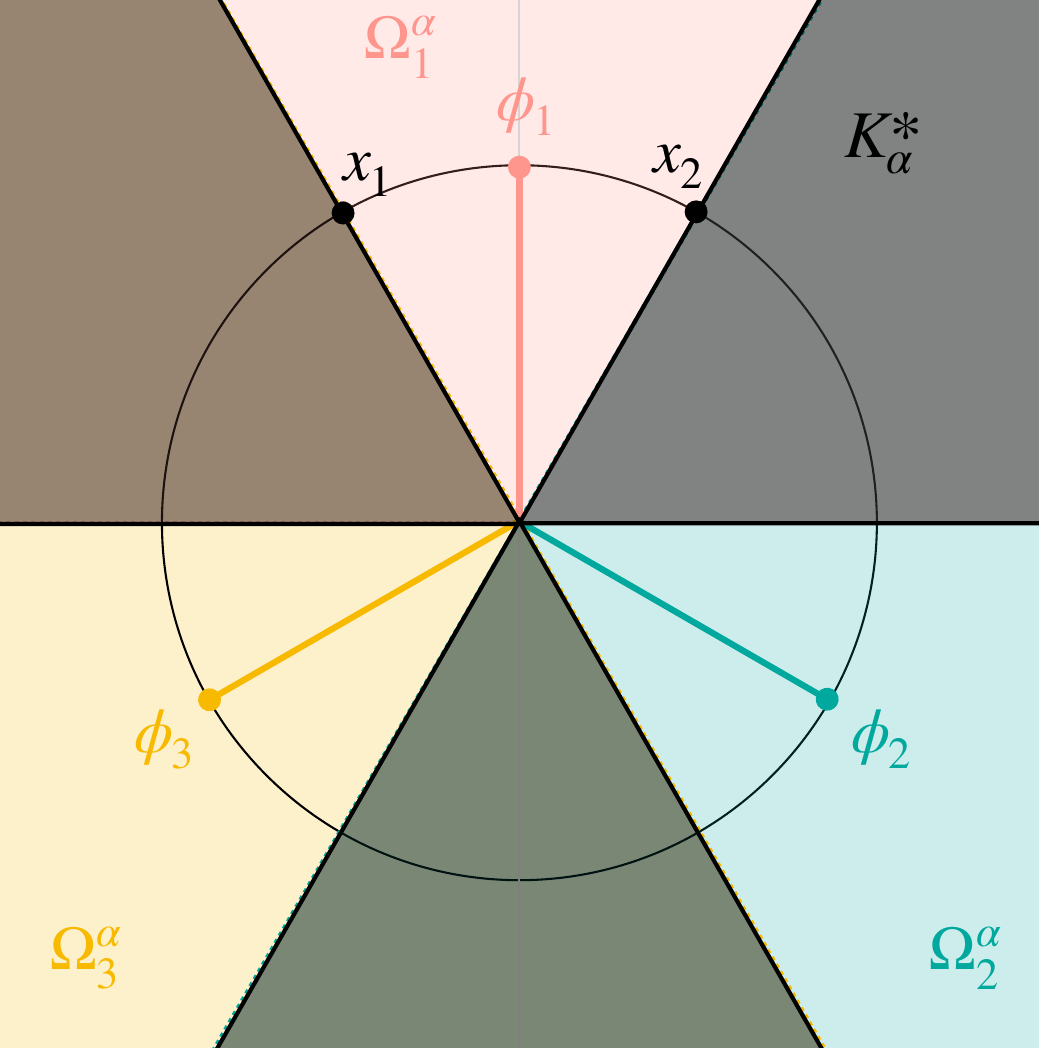}
        
    \end{subfigure}%
    \hfill
    \begin{subfigure}[t]{0.31\textwidth}
        \centering
        \includegraphics[width=\textwidth]{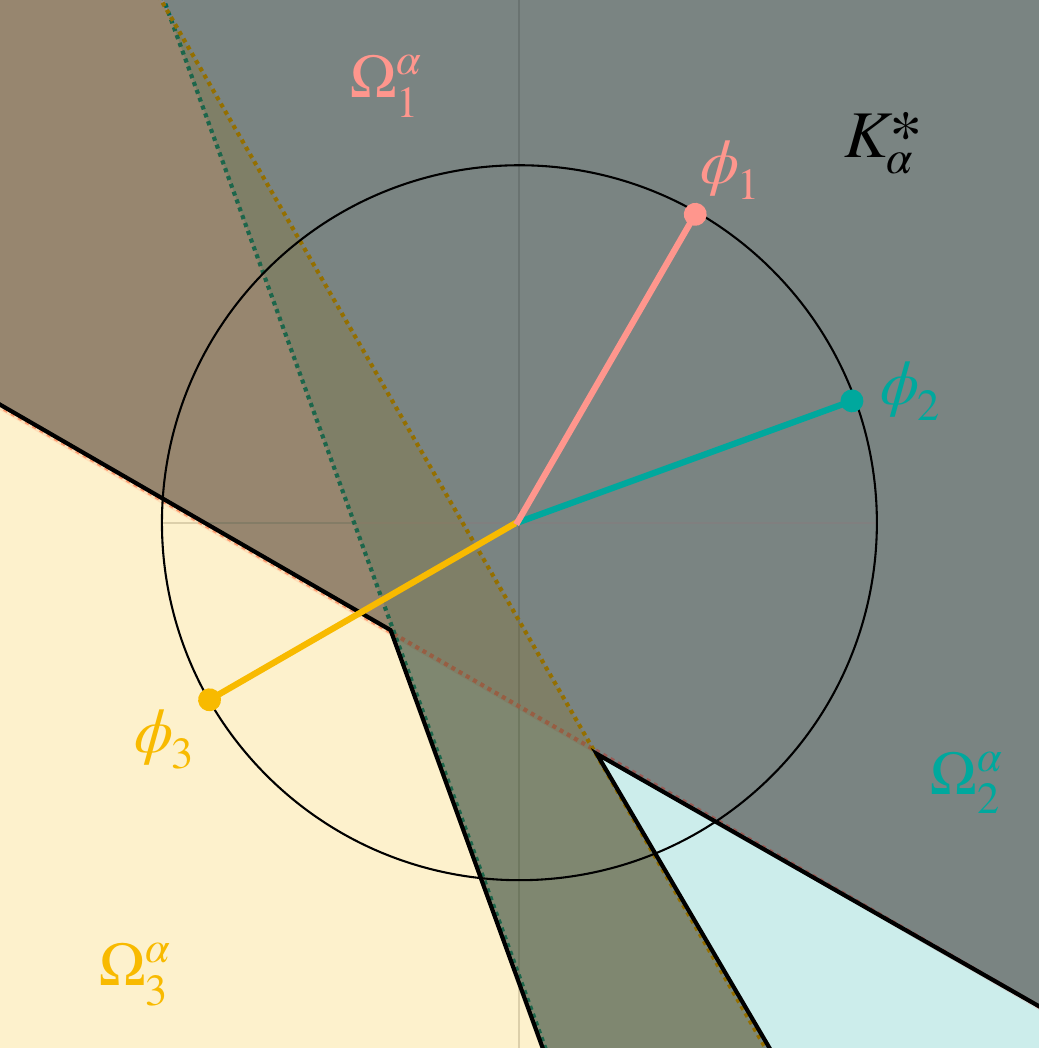}
    \end{subfigure}%
    \hfill
    \begin{subfigure}[t]{0.31\textwidth}
        \centering
        \includegraphics[width=\textwidth]{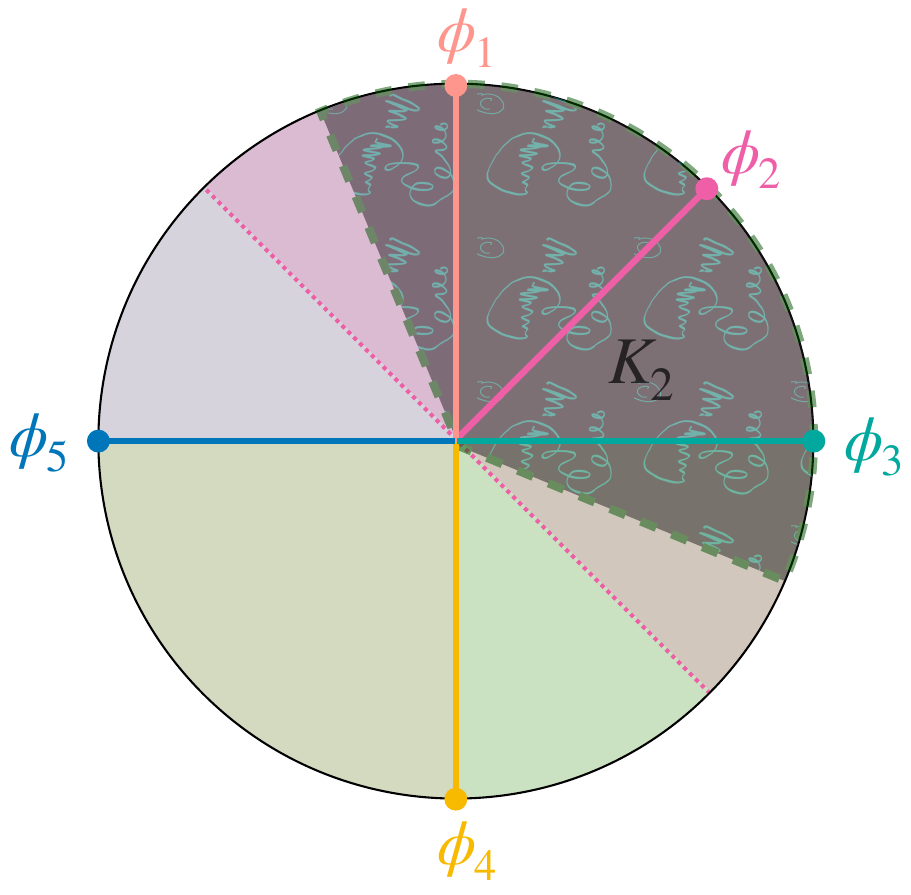}
    \end{subfigure}%

    \caption{The dark areas in the left and mid picture indicate the maximal domains $\mathcal{K}_{\alpha}^*$ for the triangle frame with zero bias (left), and a normalized random frame with random bias (mid). The right illustration corresponds to Example \ref{ex:a}. We point out how $K_2 = \{x\in K: 2\in J^*(x)\}$ looks like, where $J^*(x)$ is the most correlated basis for $x$, see Definition \ref{def:mostcorr}.}
    \label{fig:fig3}
\end{figure}

Using Theorem \ref{thm:bas} we find that any normalized frame $\Phi \subset \mathbb{S}$ is $\alpha$-rectifying on the closed ball $\mathbb{B}_r$ if and only if 
$$r\leq \inf_{x\in \RR^n\setminus\mathcal{K}_{\alpha}^*} \Vert x \Vert.$$
Another consequence of the theorem, together with $(ii)$ of Lemma \ref{prop:K} is that $\Phi$ is $\mathbf{0}$-rectifying on $\mathbb{B}_r$ if and only if $\Phi$ is $\mathbf{0}$-rectifying on $\mathbb{R}^n$. Hence, in this setting, checking a small neighborhood around the origin is already sufficient for the entire space. Note that the implication does not hold for $\alpha < \mathbf{0}$. We refer to Example \ref{ex:mb} for an example.\\

The characterization in Theorem \ref{thm:bas} further allows extending the implication from the $\alpha$-rectifying property to the injectivity of $\Ta$ (Theorem \ref{reluinj1}) to domains that are not open or convex, such as the sphere $\mathbb{S}$, the donut $\mathbb{D}_{r,s}$, and discrete data sets.

\begin{corollary}
    Let $\Phi=(\phi_i)_{i\in I}\subset \mathbb{R}^n$, $\alpha \in \mathbb{R}^m$ and $\emptyset\neq K\subseteq O \subseteq \mathcal{K}_{\alpha}^*$ for $O$ open or convex. If $\Phi$ is $\alpha$-rectifying on $K$, then $C_\alpha$ is injective on $K$.
\end{corollary}

\begin{proof}
    By Theorem \ref{thm:bas},
    $\Phi$ is $\alpha$-rectifying on $O$. Since $O$ is open or convex, by Theorem \ref{reluinj0}, $C_\alpha$ is injective on $O$, hence also on $K\subseteq O$.
\end{proof}
Note that, in general, the set $\mathcal{K}_{\alpha}^*$ is neither open nor convex. We can use this to demonstrate that a violation of the assumptions on $K$ in Theorem \ref{reluinj1} (i.e., not open and not convex) indeed leads to the conclusion that $\Ta$ is not injective. 

\begin{ex}\label{ex:mb2}
    Consider the frame $\Phi_3$ defined in Example \ref{ex:mb}, and look at
    $$ x_1 =
    \begin{pmatrix}
    \nicefrac{1}{2} \\
    \nicefrac{\sqrt{3}}{2}
    \end{pmatrix},\;
    x_2 =
    \begin{pmatrix}
    -\nicefrac{1}{2} \\
    \nicefrac{\sqrt{3}}{2}
    \end{pmatrix},
    $$
    (see Figure \ref{fig:fig3} left). For $\alpha = \mathbf{0}$ we have that $x_1, x_2\in \mathcal{K}_\alpha^*$ but also that $C_\alpha x_1 = C_\alpha x_2$. Hence, by Theorem \ref{thm:bas}, $\Phi$ is $\alpha$-rectifying on $\mathcal{K}_\alpha^*$ but $C_\alpha$ is not injective on $\mathcal{K}_\alpha^*$.
\end{ex}
A similar example can be constructed for the standard basis in $\RR^n$, $\alpha= \mathbf{0}$, and $K=\mathbb{B}_r^+$. See Figure \ref{fig:fig2} (mid) for an illustration in $\RR^2$.
The geometric intuition from the construction of the maximal domain reveals a natural trade-off to the bias vector, where 
\begin{equation*}
    \alpha'\geq \alpha\quad \Rightarrow\quad \mathcal{K}_{\alpha'}^* \subseteq 
\mathcal{K}_{\alpha}^*.
\end{equation*}
This should serve as the linking idea to the fact that finding a maximal bias for the $\alpha$-rectifying property can reveal another perspective to Theorem \ref{thm:bas}. With this, we proceed to answer Question \textbf{\textit{Q2}}.

\subsection{Maximal bias}\label{sec:maxa}
We aim to construct a maximal bias for a given frame $\Phi$ and domain $K$. This provides a characterization of the $\alpha$-rectifying property which is particularly suitable for verifying it in applications as it is straightforward to implement numerically.
Our approach to this is to decompose a frame $\Phi$ into sub-frames with highly correlated frame elements and identify the smallest analysis coefficients among all points $x\in K$ associated with these sub-frames. We present two approaches for such a decomposition. Approach A is based on finding the $n$ most correlated elements of $\Phi$ for each $x\in K$. It allows us to identify the maximal bias and with this the characterization of the $\alpha$-rectifying property of $\Phi$ under reasonable assumptions. Approach B, first introduced in \cite{haider2023relu}, is based on the vertex-facet configuration of the inscribing polytope associated with $\Phi$. It gives a geometrically intuitive sufficient condition for the $\alpha$-rectifying property of $\Phi$ but yields the maximal bias only in special situations. Algorithmic solutions are provided along with the theoretical results.

\subsection*{Approach A: Most correlated bases}
The construction of the maximal bias is based on the idea of finding the least correlated frame element in the \textit{most correlated basis} among all $x\in K$.
\begin{definition}\label{def:mostcorr}
    Let $\Phi$ be a frame and $x\in K$. We call $\Phi_{J^*(x)}$
    a most correlated basis for $x$
    if $J^*(x)\subseteq I$ satisfies that for all $J\subseteq I$ such that $\Phi_J$ is a basis it holds that
    \begin{equation}\label{eq:most}
        \min_{j\in J}\ \langle x,\phi_j\rangle \leq \min_{j\in J^*(x)}\ \langle x,\phi_j\rangle.
    \end{equation}
    We say that $\Phi$ includes a unique most correlated basis everywhere if $J^*(x)$ is unique for every $x\in K$.
\end{definition}
To give an illustration, in the setting of Example \ref{ex:mb2} we have that $J^*(x_1)=\{1,3\}$ and $J^*(x_2)=\{1,2\}$. Alternatively, we may interpret the condition in \eqref{eq:most} in the sense that $J^*(x)$ maximizes the functional
\begin{align}\label{eq:corrbas}
    \alpha(x) = \underset{\substack{J\subseteq I\\ \Phi_J\ \text{basis}}}{\max}\ \min_{j\in J}\ \langle x,\phi_j\rangle.
\end{align}
As a preliminary stage, we construct a maximal \textit{constant} bias.
This construction is similar to the one for the critical saturation level in \cite{alharbi2024sat}.
\begin{proposition}\label{prop:maxa}
    Let $\Phi$ be a frame and $K\subseteq \RR^n$. The maximal constant bias for $\Phi$ and $K$ is given by $\boldsymbol{\alpha_c}$ with
    \begin{equation}\label{eq:consta}
        \alpha_{c} = \inf_{x\in K}\underset{\substack{J\subseteq I\\ \Phi_J\ \text{basis}}}{\max}\ \min_{j\in J}\ \langle x,\phi_j\rangle = \inf_{x\in K}\ \min_{j\in J^*(x)}\ \langle x,\phi_j\rangle.
    \end{equation}
    In other words, $\Phi$ is $\boldsymbol{r}$-rectifying on $K$ if and only if $r\leq \alpha_c$.
\end{proposition}

\begin{proof}
    First, we show that $\Phi$ is $\boldsymbol{\alpha_c}$-rectifying on $K$. Let $x\in K$ then there is a basis $\Phi_{J(x)}$ such that for all $j\in J(x)$
    $$\langle x,\phi_j \rangle \geq \min_{j\in J(x)}\langle x,\phi_j \rangle \geq \alpha_c.$$
    Since $\Phi_{J(x)}$ is a frame, $\Phi$ is $\boldsymbol{\alpha_c}$-rectifying on $K$.
    
    Now let $r\in \RR$ and assume that $\Phi$ is $\boldsymbol{r}$-rectifying on $K$. For any $x\in K$ we deduce
    \begin{align}
        r\leq \inf_{x\in K} \min_{j\in I_x^{\boldsymbol{r}}}\langle x,\phi_j \rangle
        \leq
        \inf_{x\in K} \underset{\substack{j\in J\subseteq I_x^{\boldsymbol{r}}\\ \Phi_J\ \text{basis}}}{\min}\langle x,\phi_j\rangle
        \leq
        \inf_{x\in K}\underset{\substack{J\subseteq I\\ \Phi_J\ \text{basis}}}{\max}\ \min_{j\in J}\ \langle x,\phi_j\rangle = \alpha_c.
    \end{align}
\end{proof}
To construct a (non-constant) maximal bias vector we restrict ourselves to frames that include a unique most correlated basis everywhere.
Similar to the full-spark assumption, this is a mild condition in a numerical setting since for any frame there is an arbitrarily small perturbation such that the resulting perturbed frame includes a unique most correlated basis everywhere (c.f. Lemma \ref{lem:perturbation}).
If the frame is full-spark then the most correlated basis for $x$ is given by the collection of the $n$ frame elements which have the largest frame coefficients with $x$. A random frame fulfills this condition with probability one. 
Under this assumption, we can partition $K$ uniquely into subsets that are associated with a frame element that belongs to a most correlated basis. For every $i\in I$ we denote the corresponding set by
\begin{equation}\label{eq:partition1}
    K_i = \{x\in K: i\in J^*(x)\}.
\end{equation}
Since every $x\in K$ has a most correlated basis $\bigcup_{i\in I}K_i=K$ indeed holds. The right picture in Figure \ref{fig:fig3} illustrates the set $K_i$ in $\RR^2$, and the right plot in Figure \ref{fig:poly} illustrates the decomposition of $\mathbb{S}$ in $\RR^3$ into $K_i$'s. By minimizing the frame coefficients of $x\in K_i$ similar to \eqref{eq:consta} we indeed obtain a bias such that $\Phi$ possesses the $\alpha$-rectifying property but it is in general not maximal.

\begin{proposition}\label{prop:maxa2}
    Let $\Phi$ be a frame that includes a unique most correlated basis everywhere, and let $K\subseteq \RR^n$. If $\alpha^{\flat}_{K}$ is given as
    \begin{equation}\label{eq:maxa}
        \left(\alpha^{\flat}_{K}\right)_i = \inf_{x\in K_i} \langle x,\phi_i\rangle,
    \end{equation}
    then $\Phi$ is $\alpha^{\flat}_{K}$-rectifying on $K$.
\end{proposition}
\begin{proof}
    Let $x\in K$ then $\langle x,\phi_j \rangle \geq (\alpha^{\flat}_{K})_j$ for all $j\in J^*(x)$. Since $\Phi_{J^*(x)}$ is a basis, $\Phi$ is $\alpha^{\flat}_{K}$-rectifying on $K$.
\end{proof}
Note that if $K_i$ is empty then $\phi_i$ is never contained in a most correlated basis. This means that $\phi_i$ is irrelevant for the $\alpha$-rectifying property of $\Phi$, i.e., the corresponding bias can be chosen arbitrarily large without affecting it. We shall exclude these cases from the estimation for a maximal bias.

Moreover, although $\alpha^{\flat}_{K}$ gives a simple and intuitive indication of the critical bias, there is still room for increasing this bias while maintaining the $\alpha$-rectifying property. Instead of minimizing over $K_i$ as in Proposition \ref{prop:maxa}, we shall minimize over the set of all points $x$ such that $i\in J^*(x)$ and no element outside the most correlated basis is active for $x$ and $\alpha^\flat_K$. This set is given as
\begin{equation}\label{eq:Ksharp}
    K_i^\sharp = K_i\setminus \Bigg( \bigcap_{y\in K_i} \bigcup_{j\notin
    J^*(y)} \Omega_j^{\alpha^{\flat}_{K}}\Bigg).
\end{equation}
Since $K_i^\sharp$ is a subset of $ K_i$ the bias values that we get from minimizing over the $K_i^\sharp$ will be larger than the ones of $ \alpha^{\flat}_{K}$.
However, note that if 
$K_i\subseteq \big(\bigcap_{y\in K_i} \bigcup_{j\notin
J^*(y)} \Omega_j^{\alpha^{\sharp}_{K}}\big)$ then whenever $\phi_i$ belongs to the most correlated basis for $x$, there is additionally another active frame element from outside the most correlated basis.
Similarly to when $K_i=\emptyset$, the frame element $\phi_i$ can then be interpreted as being redundant in the sense that it can be completely removed from $\Phi$ while preserving the $\alpha$-rectifying property.
We give an example of such a pathological situation.
\begin{ex}\label{ex:a}
    For the frame
    $$\Phi=\left(
        \begin{pmatrix}
        1 \\
        0
        \end{pmatrix},
        \begin{pmatrix}
        \nicefrac{1}{\sqrt{2}} \\
        \nicefrac{1}{\sqrt{2}}
        \end{pmatrix},
        \begin{pmatrix}
        0 \\
        1
        \end{pmatrix},
        \begin{pmatrix}
        -1 \\
        0
        \end{pmatrix},
        \begin{pmatrix}
        0 \\
        -1
        \end{pmatrix}
        \right)$$
    and $K=\mathbb{B}$ we have that
    $\alpha_\mathbb{B}^{\flat} = \mathbf{0}$
    and
    $$K_2 = \left\{x=s\cdot \begin{pmatrix}\cos{t} \\ \sin{t} \end{pmatrix}: t\in \left[-\frac{3 \pi}{8},\frac{3 \pi}{8}\right], s\in \left[0,1\right]\right\}.$$
    The right picture in Figure \ref{fig:fig3} shows this setting. For every $x\in K_2$ there is an element outside the most correlated basis that is additionally active.
    It follows that $K_2^{\sharp}=\emptyset$.
    Hence, $\phi_2$ is redundant for the $\alpha$-rectifying property of $\Phi$.
\end{ex}
Hence, to guarantee that a maximal bias exists we shall assume $K_i^{\sharp}\neq \emptyset$ for all $i\in I$.
\begin{definition}[]\label{def:uniquemaxa}
    We call $K\subseteq \RR^n$ to be bias-exact for $\Phi$ if $K_i^{\sharp}\neq \emptyset$ for all $i\in I$, where $K_i^{\sharp}$ is defined as in \eqref{eq:Ksharp}.
\end{definition}
The following theorem contains the main result on the maximal bias and represents the counterpart to Theorem \ref{thm:bas} on the maximal domain.

\begin{theorem}[Maximal bias]\label{thm:maxa}
    Let $\Phi$ be a frame that includes a unique most correlated basis everywhere and $K\subseteq \RR^n$ be bias-exact for $\Phi$. The maximal bias for $\Phi$ and $K$ is given by $\alpha^{\sharp}_{K}$ with
    \begin{equation}\label{eq:maxasharp}
        \left(\alpha^{\sharp}_{K}\right)_i =
        \inf_{x\in K_i^\sharp} \langle x,\phi_i\rangle.
    \end{equation}
    In other words, $\Phi$ is $\alpha$-rectifying on $K$ if and only if $\alpha \leq \alpha^{\sharp}_{K}$.
\end{theorem}

\begin{proof}    
    For the converse direction, we have that $\langle x,\phi_j \rangle\geq (\alpha^{\sharp}_{K})_j$ for all $j\in J^*(x)$ and $x\in K_j^{\sharp}\neq \emptyset$. Since $\Phi_{J^*(x)}$ is a basis,
    $\Phi$ is $\alpha^{\sharp}_{K}$-rectifying on $K$.
    
    We show the implication direction by counterposition. Let $i\in I$ and assume that $\Phi$ is $\alpha$-rectifying on $K$ where $\alpha$ is given by $\alpha_i = (\alpha^{\sharp}_{K})_i + \varepsilon$ for some $\varepsilon > 0$
    and $\alpha_j=(\alpha^{\sharp}_{K})_j$ for $j\neq i$.
    By definition of the infimum in \eqref{eq:maxasharp}, the set $K_i^\sharp$ \eqref{eq:Ksharp}, and the construction and uniqueness of the most correlated basis \eqref{eq:corrbas} there is $x_0\in K_i$ such that
    \begin{equation}\label{eq:boom}
        \left(\alpha^{\sharp}_{K}\right)_i < \langle x_0,\phi_i\rangle 
        < \left(\alpha^{\sharp}_{K}\right)_i + \varepsilon\qquad \text{and}\qquad \langle x_0,\phi_j\rangle 
        < \left(\alpha^{\sharp}_{K}\right)_j
    \end{equation}
    for all $j\notin 
    J^*(x_0)$.
    This implies that the active coordinates for $x_0$ and $\alpha$ are exactly given by $$I_{x_0}^{\alpha} = J^*(x_0)\setminus \{i\}.$$
    As a consequence, $\Phi_{I_{x_0}^\alpha}$ is not a frame.
    Therefore, $\Phi$ is not $\alpha$-rectifying on $K_i$, finishing the proof.
\end{proof}
Note that while assuming uniqueness of the most correlated basis everywhere is natural in the numerical setting, it excludes frames that exhibit certain symmetries, such as the frame from Example \ref{ex:a}. In such a case, the computations of the biases in \eqref{eq:maxa} and \eqref{eq:maxasharp} will in general depend on the choice of the most correlated basis and therefore give ambiguous results. By including a condition that chooses one of the most correlated bases, Proposition \ref{prop:maxa2} and Theorem \ref{thm:maxa} can also be formulated without the uniqueness everywhere assumption. We will not pursue this idea here.\\

Summarizing, Proposition \ref{prop:maxa2} provides a simple and intuitive sufficient condition for the $\alpha$-rectifying property on $K$ via the bias vector $\alpha^{\flat}_{K}$. On the other hand, Theorem \ref{thm:maxa} provides a full characterization via the more complicated bias vector $\alpha^{\sharp}_{K}$ under some additional assumptions that are difficult to check in practice. In fact, except for special situations, it is unclear how to compute both of the presented bias vectors explicitly. To implement the construction of a bias for verifying injectivity in applications, we therefore present an algorithmic approach that computes $\alpha^{\flat}_{X_N}$ for a finite sampling set $X_N\subset K$ as an approximation for $\alpha^{\flat}_{K}$.\\

\textbf{Algorithmic solution for Approach A.} We show how $\alpha_K^{\flat}$ can be approximated numerically via sampling. Let $X_N=(x_k)_{k=1}^N\subset K$ be a sequence of $N$ samples in $K$. By iteratively updating the values of the bias estimation corresponding to the most correlated basis of the current sample as in Proposition \ref{prop:maxa2}, we obtain an approximation of $\alpha_K^{\flat}$ through the sampling set $X_N$.
To measure the sampling error quantitatively, we use the Euclidean covering radius (or ''mesh norm‘‘) of $X_N$ for $K$ \cite{damelin2005point}, defined by
\begin{equation}\label{eq:cover}
    \rho(X_N; K) = \sup_{x\in K} \min_{1\leq i \leq N}\Vert x - x_i\Vert.
\end{equation}
\begin{theorem}[Sampling-based Bias Estimation]\label{thm:samp}
    Let $X_N=\{x_i\}_{i=1}^N\subset K$ and $\Phi\subset \mathbb{S}$ be a normalized full-spark frame. Choose $\alpha^{(0)} \in \RR^m$ and iteratively define for all $1\leq k\leq N$ and $i\in J^*(x_{k})$
    \begin{equation}\label{eq:mcbe}
        (\alpha^{(k)})_{i} = \min \left\{ \langle x_k,\phi_i\rangle, \alpha^{(k-1)}_i \right\}.
    \end{equation}
    Then 
    $\Phi$ is $\alpha^{(N)}$-rectifying on $X_N$. Moreover, $\Phi$ is $(\alpha^{(N)}-\rho(X_N;K))$-rectifying on $K$.
\end{theorem}
If the elements in $\Phi$ are not normalized, we can extend the statement above by including the norms $w = (\Vert \phi_i \Vert)_{i\in I}$, obtaining that $\Phi$ is $(\alpha^{(N)}-\rho(X_N;K)\cdot w)$-rectifying on $K$.
If the frame elements of $\Phi$ lie in $K$, a good initialization is starting the bias estimation with the frame elements themselves as samples. Otherwise, we may set $(\alpha^{(0)})_i = \infty$ for all $i\in I$.

\begin{proof}[Proof of Theorem \ref{thm:samp}]
    Let $x\in K$, then there is $x_i\in X_N$ with $\Vert x-x_i \Vert\leq \rho(X_N;K)$.
    Furthermore, for every $j\in I_{x_i}^{\alpha^{(N)}}$ it holds that
    \begin{align}\label{eq:rad}
    \begin{split}
        (\alpha^{(N)})_j &\leq \langle x_i,\phi_j \rangle\\ &= \langle x_i-x,\phi_j \rangle + \langle x,\phi_j \rangle \\&\leq \Vert x-x_i \Vert + \langle x,\phi_j \rangle \\&\leq \rho(X_N;K) + \langle x,\phi_j \rangle.
    \end{split}
    \end{align}
    By rearranging \eqref{eq:rad} it follows that $\langle x,\phi_j \rangle \geq (\alpha^{(N)})_j - \rho(X_N;K)$. We deduce that $I_{x_i}^{\alpha^{(N)}} \subseteq I_{x}^{\alpha^{(N)} - \rho(X_N;K)}$ for all $x\in K$.
    Clearly, $\Phi$ is $\alpha^{(N)}$-rectifying on $X_N$ by construction. Using \eqref{eq:rad}, the second claim follows immediately.
\end{proof}

Implementing the algorithm described in Theorem \ref{thm:samp} can be done via a Monte-Carlo approach, where $X_N$ is a collection of $N$ random samples on $K$ w.r.t.~some probability measure on $K$, see Appendix C1. We demonstrate numerical experiments in Section \ref{sec:experiments}. Intuitively, we want a measure that guarantees a small covering radius $\rho(X_N; K)$ for large $N$, such that $\alpha^{(N)}$ converges to $ \alpha_K^{\flat}$ in probability. The uniform distribution on $\mathbb{B}_r$ is a possible example.
If $X_N$ are i.i.d.~uniform samples on $\mathbb{S}$ the expectation of $\rho(X_N;\mathbb{S})$ is given by
\begin{equation}\label{eq:asym}
    \mathbb{E}[\rho(X_N;\mathbb{S})] \asymp \left(\frac{\log(N)}{N}\right)^{\frac{1}{n}}.
\end{equation}
The above estimate and more explicit tail-bound estimates for the probability distribution of $\rho(X_N;\mathbb{S})$ are derived in \cite{reznikov2016covering}.
Unfortunately, this indicates that in high dimensions it becomes infeasible to handle the covering radius only by increasing the number of test samples $N$. Hence, it appears necessary to construct the sampling sequence $X_N$ in a more structured way, e.g., by quasi Monte-Carlo methods \cite{breger2018covering}, or by sampling directly from the distribution of the dataset. With the latter approach, the number of sampling points for a good approximation of $\alpha_K^{\flat}$ can potentially be reduced, and the injectivity of the ReLU layer is ensured on a domain that is tailored to the dataset.

\subsection*{Approach B: Facets of the inscribing polytope}
We obtain another natural decomposition of $\Phi$ into sub-frames with high correlation via the convex polytope that arises from taking the convex hull of the set of all elements in $\Phi$ \cite{haider2023relu}. This so-called \textit{inscribing polytope} of $\Phi$ \cite{richardson2009inscribing} is given by
\begin{equation}\label{eq:poly}
    P_\Phi=\{x\in \RR^n: x = \sum_{i \in I} c_i \cdot \phi_i, c_i\geq 0, \sum_{i \in I} c_i = 1 \}.
\end{equation}
Any non-empty intersection of $P_\Phi$ with an affine half-space such that none of the interior points of $P_\Phi$ (w.r.t.~the induced topology on $P_\Phi$) lie on its boundary is called a face of $P_\Phi$ \cite{zieg12}. The $0$-dimensional faces of $P_\Phi$ are known as vertices and the $(n-1)$-dimensional faces are called \textit{facets}.
Assuming normalized frames here ($\Phi\subset\mathbb{S}$) the set of vertices of $P_\Phi$ always coincides with the set of frame elements of $\Phi$. Note that this is generally the case if the elements in $\Phi$ lie in a strictly convex set. Moreover, every facet is a convex polytope where the vertices coincide with a sub-collection of $\Phi$, and every element occurs as a vertex at least once. We refer to Figure \ref{fig:fig4} for an illustration in $\RR^3$.
For a facet $F$, we denote the index set corresponding to its vertices by
\begin{equation}\label{eq:vertex}
    I_{F}=\{i\in I:\phi_i\in F\}.
\end{equation}
The following property is key, stated and proven in \cite{haider2023relu}.
\begin{lemma}\label{lem:facet}
    Let $\Phi$ be a frame and $F$ be a facet of $P_\Phi$. If $0\notin F$, then $\Phi_{I_{F}}$ is a frame.
\end{lemma}
This ensures that the facets of $P_\Phi$ provide a natural decomposition into sub-frames $\Phi_{I_{F}}$ of $\Phi$. Moreover, for any facet $F$, there is $a\in \RR^n$, $a \neq 0$ and $b\in \RR$ such that $F=\{x\in P_\Phi:\langle a,x\rangle=b\}$, and therefore,
\begin{align*}
    \langle a,\phi_k \rangle& = b,\; \text{for } k \in I_{F},\\
    \langle a,\phi_\ell \rangle &< b,\; \text{for }\ell \notin I_{F}.
\end{align*}
Hence, the vertices of any facet are a frame that is highly correlated to all points ``close'' to the facet. In particular, $\Phi_{I_{F}}$ is the most correlated basis for the normal vector $a$.

\begin{figure}[t]
    \centering
    \includegraphics[width=\linewidth]{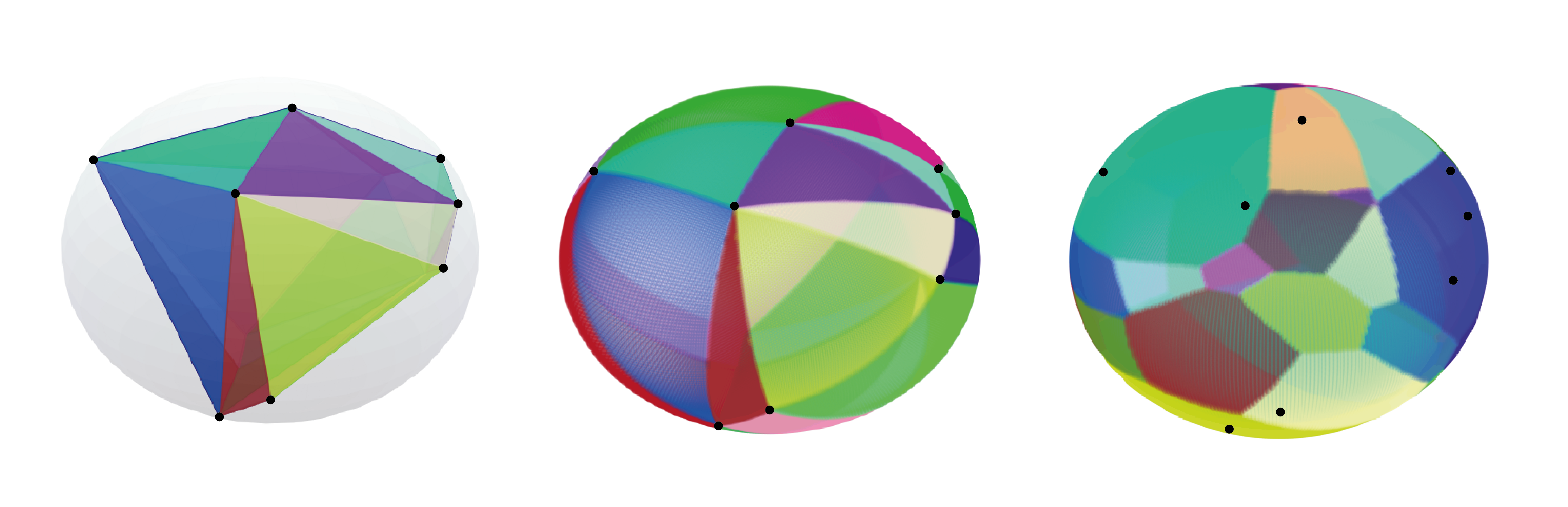}
  \caption{The three subplots show different decompositions of the sphere in $\RR^3$. The black dots indicate the $m=12$ frame elements of a random frame on $\mathbb{S}$. From left to right: The facets $F_j$ of the inscribing polytope $P_\Phi$, the associated spherical caps $F_j^\mathbb{S}=\operatorname{cone}(F_j)\cap \mathbb{S}$, and the spherical patches associated to different most correlated bases obtained by $J^*(x)$. While the facets provide a very intuitive and simple decomposition into sub-frames, the decomposition via the most correlated bases minimizes the correlation directly.}
  \label{fig:poly}
\end{figure}

Now, analog to the decomposition of $K$ using most correlated bases \eqref{eq:partition1} we can decompose $K$ into facet-specific subsets $F_j^K$, each associated with a facet $F_j$ of $P_\Phi$ (according to some enumeration of the facets).
A natural and practical approach in this setting is a decomposition into conical caps resulting in
\begin{align}
    F_j^K = \operatorname{cone}(F_j)\cap K =\{x\in K: x = c y\; , y\in F_j, c\geq 0\}. \label{eq:cone}
\end{align}
Figure \ref{fig:poly} shows the facets $F_j$ of the inscribing polytope (left), and the decomposition of $\mathbb{S}$ into spherical caps $F_j^{\mathbb{S}}$ (mid) for a i.i.d.~randomly generated frame $\Phi$ with elements on $\mathbb{S}$.
To guarantee that $K=\bigcup_j F_j^K$ and that $0$ does not lie on any of the facets (requirement for using Lemma \ref{lem:facet}) we have to assume that $0$ lies in the interior of $P_\Phi$ (w.r.t.~the topology in $\RR^n$). The property of $\Phi$ that ensures this was introduced in \cite[Definition 1]{behr18}, where $\Phi$ is called \textit{omnidirectional}. Equivalently, we can say that there is no half-space containing all elements of $\Phi$, which can be easily verified numerically via convex optimization, see the appendix in \cite{behr18}. Moreover, every non-omnidirectional frame can be made omnidirectional by including a vector constructed as the negative normalized mean of all frame elements, see Appendix B.

\begin{remark}
    The construction of the $F_j^K$ using cones and the assumption of omnidirectionality are natural if $K$ is centered around the origin. In other situations, one might want to come up with alternative constructions of $F_j^K$ and a different notion of omnidirectionality that are more suited to the geometry $K$. In this work, we restrict ourselves to the described setting. 
\end{remark}

Assuming omnidirectionality, we can use Lemma \ref{lem:facet} to identify the minimal analysis coefficient $\langle x,\phi_i \rangle$ that can occur for any $x$ contained in any $F_j^K$ that contains $\phi_i$. This guarantees that for any $x\in F_j^K$ the sub-frame $\Phi_{I_{F_j}}$ is active. Following \cite{haider2023relu}, the procedure is called \textit{polytope bias estimation} (PBE). The following theorem generalizes the results in \cite{haider2023relu} from $\mathbb{B}_r$ and $\mathbb{B}_r^+$ to general bounded $K$.

\begin{theorem}[Polytope Bias Estimation]\label{thm:pbe}
    Let $\Phi\subset \mathbb{S}$ be a normalized omnidirectional frame and $K\subseteq \RR^n$ bounded. Then $\Phi$ is $\alpha_K^\Delta$-rectifying on $K\subseteq \RR^n$ with $\alpha_K^\Delta$ given by
    \begin{equation}\label{eq:aK}
        \left(\alpha_K^\Delta\right)_i= \inf_{
        \substack{x \in F_j^K\\
        j:\phi_i\in F_j}
        } \langle x , \phi_i \rangle.
    \end{equation}
\end{theorem}
\begin{proof}
    Since $\Phi$ is omnidirectional, for any $x\in K$ there is a facet $F_j$ such that $x\in F_j^K$. It follows from \eqref{eq:aK} that $\langle x,\phi_i \rangle \geq \left(\alpha_K^\Delta\right)_i$ for all $i\in I_{F_j}$. By Lemma \ref{lem:facet}, $\Phi_{I_{F_j}}$ is a frame, hence, $\Phi$ is $\alpha_K^\Delta$-rectifying on $K$.
\end{proof}
This procedure naturally takes the geometry of the frame into account and provides an intuitive way of estimating a large bias vector such that the frame becomes $\alpha$-rectifying. Note, however, that in general this only yields a sufficient condition. A special case where it is also necessary is discussed in Lemma \ref{lem:err}. In the following, we demonstrate how the PBE simplifies for specific concrete situations.\\

\begin{figure}[t]
    \centering
    \begin{subfigure}[t]{0.29\textwidth}
        \centering
        \includegraphics[width=\textwidth]{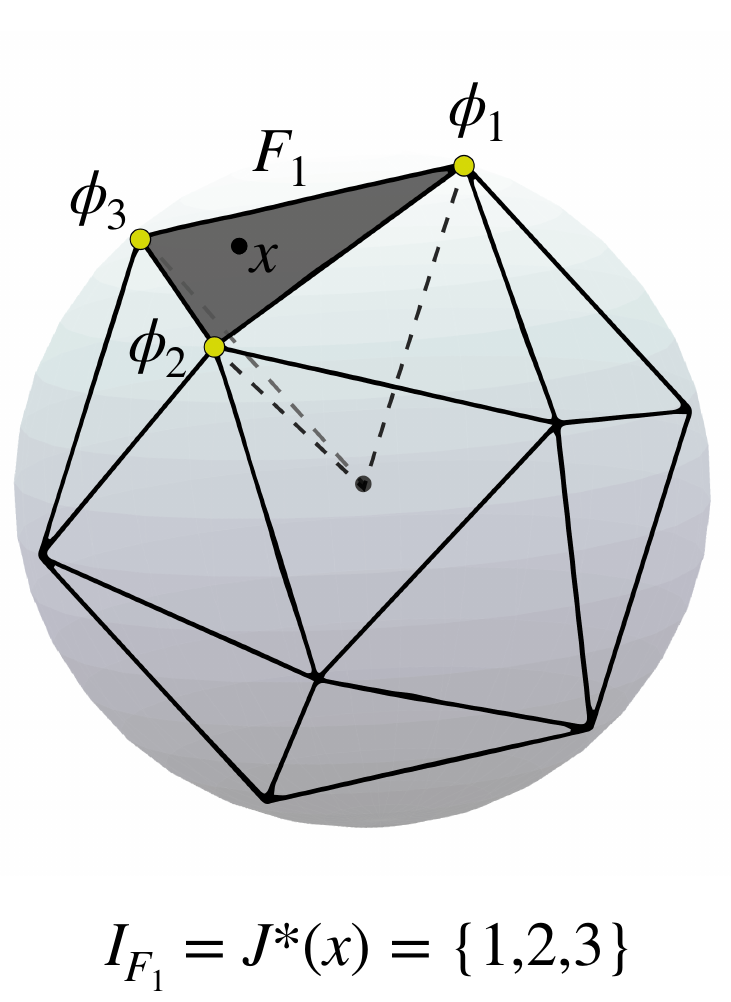}
        
    \end{subfigure}%
    \hfill
    \begin{subfigure}[t]{0.31\textwidth}
        \centering
        \includegraphics[width=\textwidth]{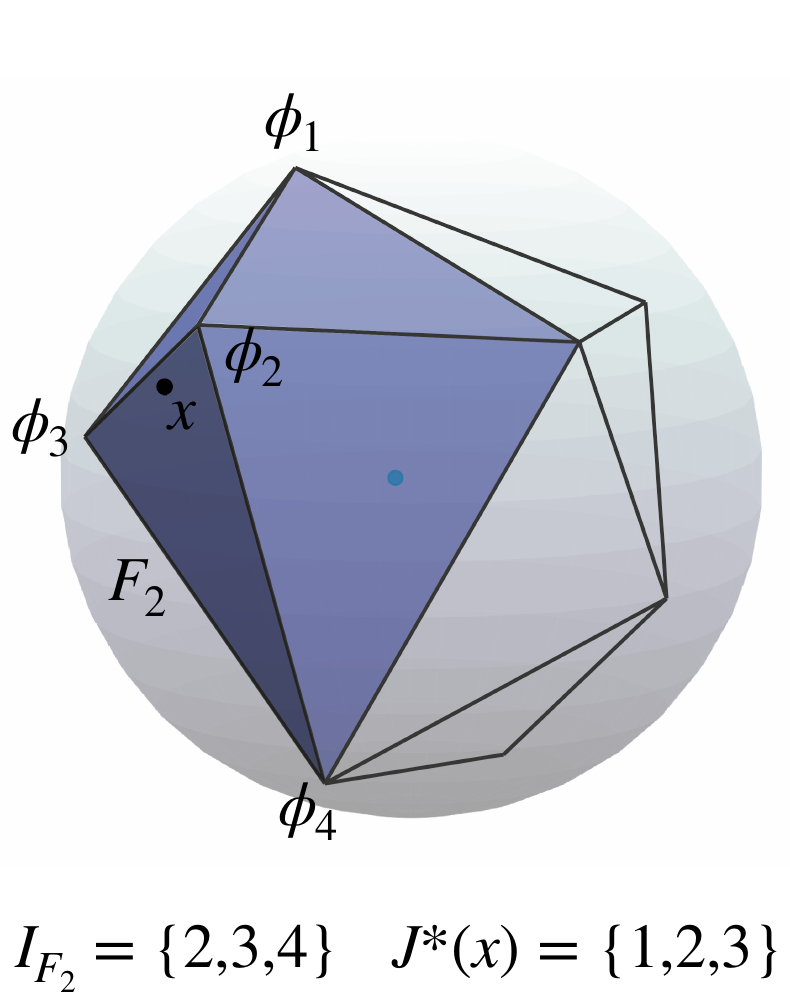}
    \end{subfigure}%
    \hfill
    \begin{subfigure}[t]{0.35\textwidth}
        \centering
        \includegraphics[width=0.85\textwidth]{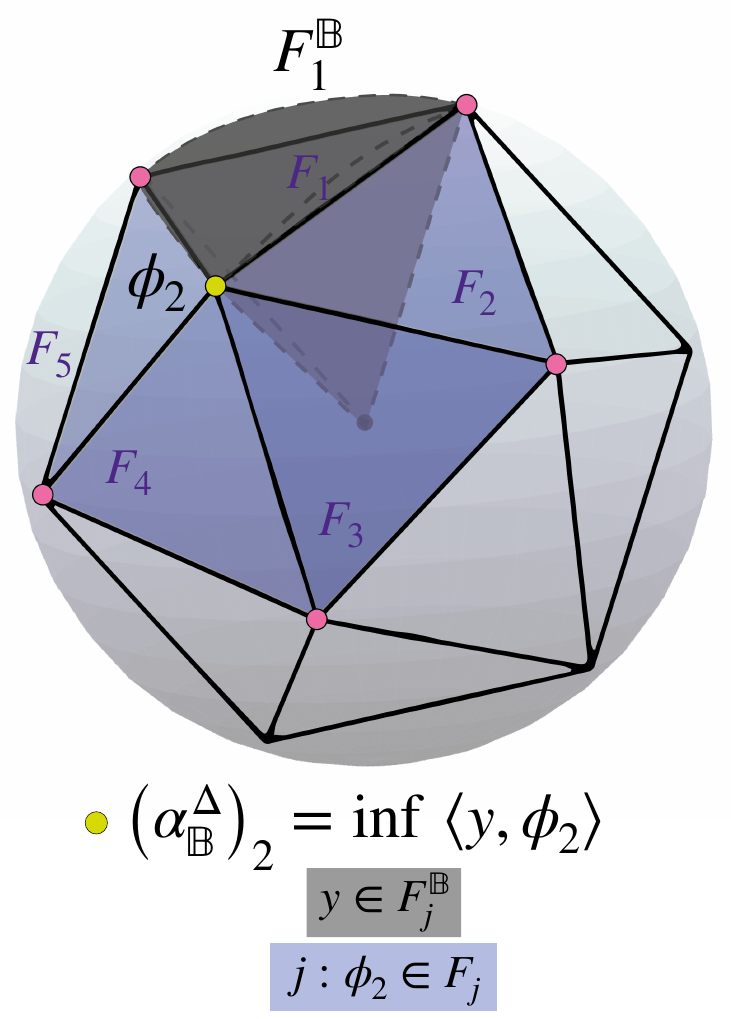}
    \end{subfigure}%

    \caption{For the Icosahedron frame we have that $x\in F_j \Leftrightarrow I_{F_j} = J^*(x)$ (left). For less regular frames, this does not hold anymore (mid). The right picture illustrates the PBE on $\mathbb{B}$ for the Icosahedron frame. To get $\left(\alpha_{\mathbb{B}}^\Delta\right)_2$ the infima are taken over the points in the conical parts (dark gray area) for all adjacent facets of $\phi_2$ (blue).}
    \label{fig:fig4}
\end{figure}

\textbf{Algorithmic solution for Approach B.} The advantage of the PBE is that for simple standard domains $K$ it is easy to compute $\alpha_K^\Delta$ via linear programs. In the following proposition, we formulate the PBE for five prototypical domains. We provide a detailed discussion and corresponding pseudo-code for implementing the corresponding optimizations in Appendix C.

\begin{proposition}\label{prop:pbeall}
    Let $\Phi\subset \mathbb{S}$ be a normalized omnidirectional frame. The following holds.
    \begin{enumerate}[(i)]
        \item $\Phi$ is $\alpha_\Phi^\Delta$-rectifying on the boundary of the polytope $\partial P_{\Phi}=\bigcup_j F_j$ with $\alpha_\Phi^\Delta$ given by
        \begin{equation*}\label{eq:pbeboundary}
            \left(\alpha_{\Phi}^\Delta\right)_i = \min_{
            \substack{\ell \in I_{F_j}\\j:\phi_i\in F_j}
            } \langle \phi_{\ell},\phi_i \rangle.
        \end{equation*}
        \item $\Phi$ is $\alpha_{\mathbb{S}}^\Delta$-rectifying on the sphere $\mathbb{S}$ with $\alpha_{\mathbb{S}}^\Delta$ given by
        \begin{equation*}\label{eq:pbeS}
            \left(\alpha_{\mathbb{S}}^\Delta\right)_i = \min\{\min_{
            \substack{x \in F_j^\mathbb{S}\\j:\phi_i\in F_j}
            } \langle x , \phi_i \rangle, \left(\alpha_{\Phi}^\Delta\right)_i\}.
        \end{equation*} 
        The inner minima are the solutions of convex linear programs. In particular, they are equal to $\left(\alpha_{\Phi}^\Delta\right)_i$ whenever they are non-negative.
        \item $\Phi$ is $(r^{-1}\cdot \alpha_\mathbb{B}^\Delta)$-rectifying on the donut $\mathbb{D}_{r,s}$ for
        $0\leq s<r$ with $\alpha_{\mathbb{B}}^\Delta$ given by
        \begin{equation*}\label{eq:pbeb}
            \left(\alpha_{\mathbb{B}}^\Delta\right)_i = \min\{s,\left(\alpha_{\mathbb{S}}^\Delta\right)_i\}.
        \end{equation*} 
        The case $s=0$ yields a bias estimation for the closed ball $\mathbb{B}_r$ (Figure \ref{fig:fig4}). 
        \item Let $J^+=\{j\in I : F_j\cap \RR^n_+\neq \emptyset\}$ and $ I^+ = \bigcup_{j\in J^+} I_{F_j}$ then $\Phi$ is $( r^{-1}\cdot \alpha_{\mathbb{B}^+}^\Delta)$-rectifying on the non-negative part of the ball ${\mathbb{B}^+_r}$ with $\alpha_{\mathbb{B}^+}^\Delta$ given by
    \begin{equation}\label{eq:a+}
        \left(\alpha_{\mathbb{B}^+}^\Delta\right)_i=
        \begin{cases}
            \left(\alpha_{\mathbb{B}}^\Delta\right)_i &\text{ for } i \in I^+\\
            s_i &\text{ else,} 
        \end{cases}
    \end{equation}
    where $s_i\in \RR$ is arbitrary.
        \item If $\alpha_{\Phi}^\Delta\geq 0$, then $\left(s \cdot \phi_i\right)_{i\in I}$ for $s \geq 0$ is $\left(s\cdot \alpha_{\Phi}^\Delta\right)$-rectifying on $(\mathring{\mathbb{B}}_s)^{\mathrm{c}}= \RR^n \setminus \mathring{\mathbb{B}}_{s}$.
    \end{enumerate}
\end{proposition}
\begin{proof}
    The points $(i)-(iv)$ are direct consequences of \cite[Theorem 4.4 and Theorem 4.6]{haider2023relu}, where detailed proofs can be found.
    To show $(v)$, note that $\left(s \cdot \phi_i\right)_{i\in I}$ has the same combinatorial facet structure as $\Phi$ and therefore it is still omnidirectional. In particular,
    $$
    (\mathring{\mathbb{B}}_s)^{\mathrm{c}} = \{x = s t y: y\in \mathbb{S}, t \geq 1 \}.
    $$
    The statement follows by $\langle s t y,\phi_i \rangle  \geq s\cdot \left(\alpha_{\Phi}^\Delta\right)_i$ for $t\geq 1$ and $s \geq 0$.
\end{proof}
Other than the Monte-Carlo sampling-based approach, this bias estimation procedure yields a deterministic sufficient condition for the $\alpha$-rectifying property, which is necessary only in special situations. 
%
In general, there are two properties of the inscribing polytope $P_\Phi$ that affect the PBE in a way that the estimated biases become smaller.
\begin{itemize}
    \item[(1)] The more vertices a facet has, the smaller the infima in \eqref{eq:aK} become. In the case where all facets have the minimal number of $n$ vertices, $P_\Phi$ is called \textit{simplicial}. It is known that a polytope with vertices that are i.i.d.~uniform samples on $\mathbb{S}$ is simplicial with probability one \cite{randpol}.
    \item[(2)] The decomposition of $K$ using cones in \eqref{eq:cone} is natural and convenient for implementation, but leads to a sub-optimal partition if $P_\Phi$ is geometrically very irregular, i.e., the sizes of the facets are significantly different. In such a scenario, for $x\in F_j^K$ where $F_j$ is a very large facet, there might be a smaller neighboring facet $F_k$ which provides larger analysis coefficients for $x$ and, hence, a better estimation. For an illustration of such a situation see the center plot in Figure \ref{fig:fig4}.
\end{itemize}
For frames with inscribing polytopes that are simplicial and regular, we can show that the PBE indeed yields a maximal bias. We set $K=\mathbb{S}$.


\begin{lemma}\label{lem:err}
    Let $\Phi\subset \mathbb{S}$ be a normalized omnidirectional frame such that $P_{\Phi}$ is simplicial and for all $i\in I$ it holds that $\langle \phi_i,\phi_k \rangle = \langle \phi_i,\phi_\ell \rangle$ for all $\phi_k,\phi_\ell$ sharing a facet with $\phi_i$.
    Then $\Phi$ is $\alpha$-rectifying on $\mathbb{S}$ if and only if  $\alpha \leq \alpha_{\mathbb{S}}^\Delta$.
\end{lemma}

\begin{proof}
    The converse direction directly follows from Theorem \ref{thm:pbe}. We show the implication direction by counterposition. At first note that the regularity condition implies that $\alpha_{\mathbb{S}}^\Delta$ is constant. Let $\alpha\in\RR^m$ be such that $\alpha_{i} > \left(\alpha_{\mathbb{S}}^\Delta\right)_{i}$ for fixed $i\in I$. There is a facet $F$ with $i\in I_{F}$ and $x^*\in F^\mathbb{S}$ satisfying $\langle x^*,\phi_{i} \rangle = \left(\alpha_{\mathbb{S}}^\Delta\right)_{i} < \alpha_{i}$. By regularity, it follows that $I_{x^*}^{\alpha}=I_{F}\setminus \{{i}\}$. Since $P_{\Phi}$ is simplicial $\Phi_{I_{F}\setminus \{{i}\}}$ is not a frame, hence, $\Phi$ is not $\alpha$-rectifying on $\mathbb{S}$.
\end{proof}
Examples for this are frames whose inscribing polytopes are convex regular polytopes in $\RR^n$. In a more general sense, we expect the PBE to be very effective for frames with evenly distributed frame elements.


\begin{remark}
    The two bias estimation procedures described in Approach A and B are fundamentally linked. For each facet $F$ of $P_\Phi$ the hole radius of $F$ is defined as the Euclidean distance from the boundary to the center of its spherical cap $F^{\mathbb{S}}$ \eqref{eq:cone}. The Euclidean covering radius \eqref{eq:cover} of $\Phi$ for $\mathbb{S}$ is the largest hole radius among all facets \cite{seri2022covering}.
\end{remark}

\subsection{Numerical experiments}\label{sec:experiments}
For a given sampling sequence $X_N$ and a full-spark assumption on $\Phi$, the numerical implementation of the sampling-based bias estimation in Theorem \ref{thm:samp} is straightforward. Similarly, there are convex hull algorithms available, such that the implementation of the cases of the PBE from Proposition \ref{prop:pbeall} is straightforward, too. Besides the pseudo-code found in the appendix, we provide concrete implementations in Python, together with the code for reproducing all experiments in this section in the accompanying repository \href{https://github.com/danedane-haider/Alpha-rectifying-frames}{https://github.com/danedane-haider/Alpha-rectifying-frames}. 

\subsubsection*{Evolution towards injectivity (Approach A)} We demonstrate the basic functionality of the Monte-Carlo sampling-based algorithm. For this, we choose the frame $\Phi$ and the sampling sequences $X_N$ to consist of i.i.d.~uniform samples on $\mathbb{B}$. For every step in the approximation of $\alpha^\flat_{\mathbb{B}}$ we measure the proportion of samples from an unseen test sampling sequence $Y_M$ (also i.i.d.~uniform on $\mathbb{B}$), for which $ \Phi$ is $\alpha^{(k)}$-rectifying. Figure \ref{fig:propinj} shows the empirical mean and variance over $1000$ independent trials of this procedure for two different dimensions, $n=3$ (left) and $n=30$ (right), and three different redundancies ($2,3.3,$ and $9$), respectively. We observe that in some configurations the associated ReLU layer is injective already after a few hundred iterations. In others, it takes up to $1000$ iterations. It is especially fast for ReLU layers in low dimensions with low redundancy.
This can be explained by the fact that for draws of $\Phi$ which yield very unevenly distributed points on $\mathbb{B}$ (which happens more likely in high dimensions with high redundancy) the iterative scheme struggles to update $\alpha^{(k)}$ efficiently, which results in the procedure taking particularly long. In general, all tested examples became injective reliably.

\begin{figure}
  \centering
  \includegraphics[width=\linewidth]{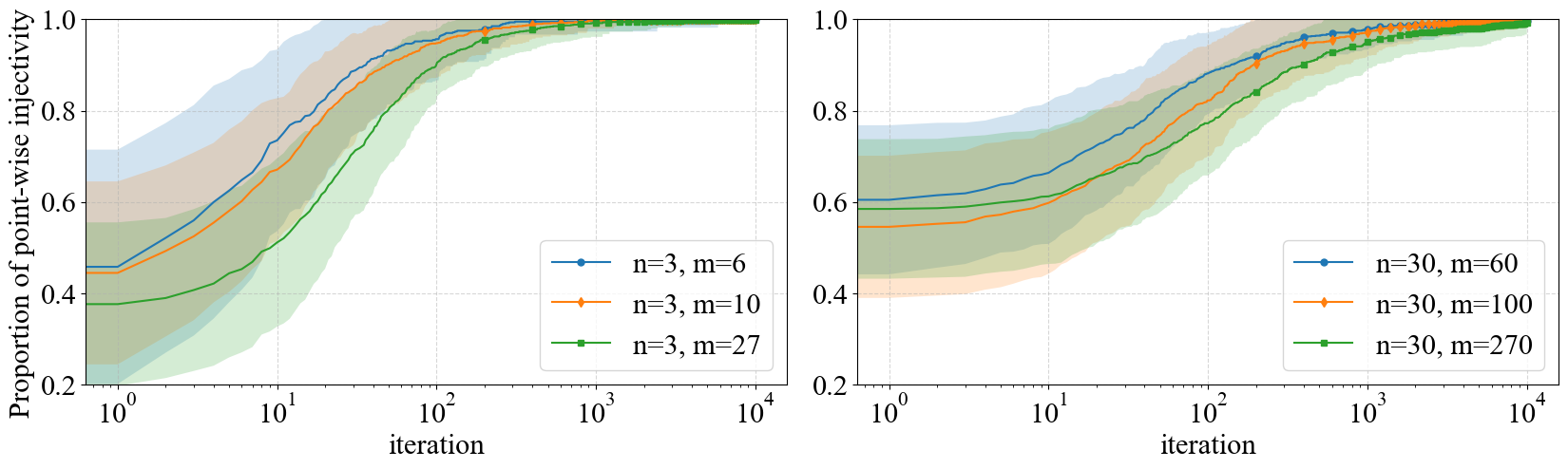}
  \caption{Per iteration $k$, the plots show the proportion of a test sample sequence $Y_M$ where the ReLU layer with bias $\alpha^{(k)}$ is injective. Left: $n=3$. Right: $n=30$; Both for redundancies $2, 3.3,$ and $9$. The ReLU layers are becoming injective reliably after about $10^4$ iterations. The procedure is fastest for low dimensions and low redundancy.}
  \label{fig:propinj}
\end{figure}

\subsubsection*{Effect of redundancy (Approach A)} We use approximations of $\alpha_{\mathbb{B}}^{\flat}$ to verify the injectivity of ReLU layers with random weights and biases systematically for different redundancies. Thereby, we numerically verify the conjecture stated in \cite{maillard2023injectivity} that the transition from non-injectivity to injectivity happens at a redundancy $q\in (6.6979, 6.6981)$ in a non-asymptotic setting. We let both, the frame $\Phi$ and the sampling sequence $X_N$ contain i.i.d.~standard normal points and compute $\alpha^{(N)}$ with $N=5\cdot 10^5$ for all redundancy settings where $2\leq n\leq 30$ and $n\leq m\leq 150$. By Theorem \ref{thm:samp}, we know that for every setting $\Phi$ is $\alpha^{(N)}$-rectifying on $X_N$. Inspired by the asymptotic expression for the covering radius in \eqref{eq:asym}, we subtract a correcting term of $\rho^*(n,N) = 0.05 \cdot \left(\frac{\log(N)}{N}\right)^{\frac{1}{n}}$ to compensate for insufficient amount of sampling in higher dimensions. This yields that for every setting we have that $\Phi$ is $\left(\alpha^{(N)}-\rho^*(n,N)\right)$-rectifying on $\mathbb{B}$ with high probability. The factor $0.05$ was chosen experimentally.

To test if the ReLU layer associated with one of the realizations of $\Phi$ is injective for a given bias $\alpha$ we have to verify that $\alpha \leq \alpha^{(N)}-\rho^*(n,N)$. We compare three settings for biases with i.i.d.~normal values with mean zero and variances,
$$(i)\ \sigma^2=0\qquad \quad (ii)\ \sigma^2=0.1\qquad \quad (iii)\ \sigma^2=1$$
Figure \ref{fig:isinj} shows the results for the three settings from left to right.
Setting $(i)$ is the one where the conjecture in \cite{maillard2023injectivity} was formulated. Looking at the solid magenta line in Figure \ref{fig:isinj} (left) we can observe that our method is capable of numerically reproducing the conjecture. On the injectivity of random ReLU layers with non-zero bias, there are no theoretical results in the literature so far. Hence, our approach yields some novel insights here. For small variance (Setting $(ii)$) we observe that the clear boundary from the previous setting blurs out. For the standard variance in setting $(iii)$ this behavior further intensifies. Note that an according change of the variances in the distribution of $\Phi$ or $X_N$ instead gives the same result. These observations show that the way the bias in a ReLU layer is initialized has a big influence on its injectivity.

\begin{figure}
    \centering
    \includegraphics[width=\linewidth]{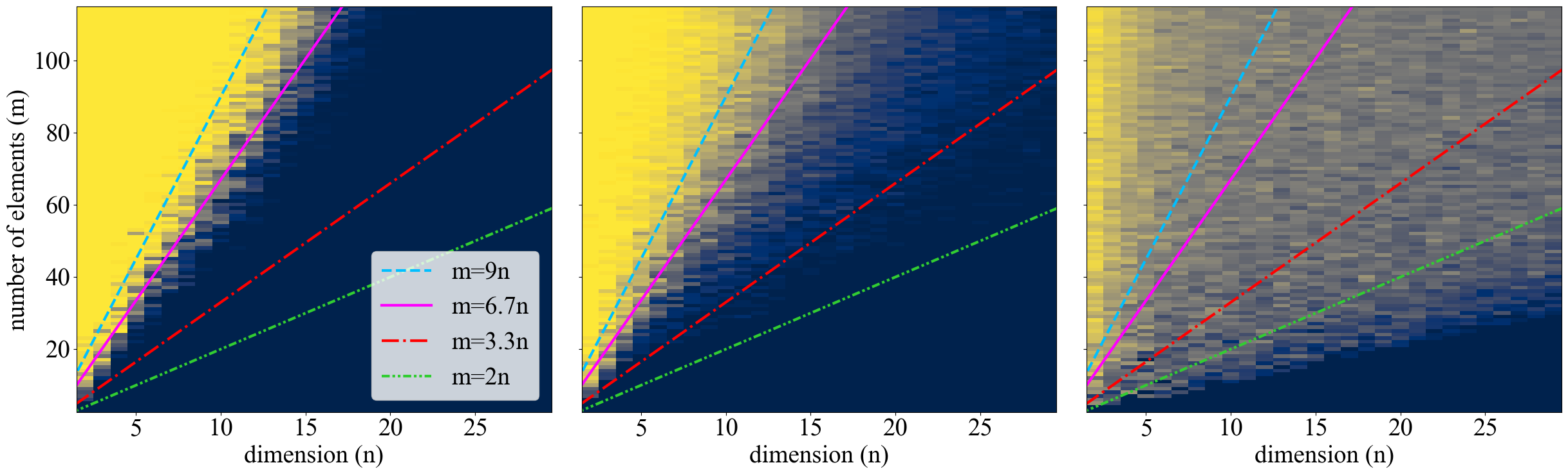}
    \caption{The plots show the injectivity behavior of random ReLU layers for different redundancies ($2\leq n\leq 30, n\leq m\leq 150$). The brightness encodes the proportion of the values of the given bias $\alpha$ that are smaller than the ones in $(\alpha^{(N)}-\rho^*(n,N))$ for $N=5\cdot 10^5$. For values of one (yellow), the ReLU layer is injective on $\mathbb{B}$. In all settings, the samples in $X_N$ are i.i.d.~standard normal and the bias i.i.d.~normal with different variances.
    From left to right: $\sigma^2=0$, $\sigma^2=0.1$, and $\sigma^2=1$. For $\sigma^2=0$ we observe the clear transition from non-injective to injective at a redundancy of $6.7$ (solid magenta line) that aligns well with the conjecture from the literature. For larger variance, the transition blurs out quickly and prevents us from predicting clear statements about injectivity.}
    \label{fig:isinj}
\end{figure}

\subsubsection*{Approximation of the maximal bias (Approach A \& B)}
In the setting of Lemma \ref{lem:err} we have shown that the PBE yields a maximal bias. This allows us to study the approximation of the sampling-based approach to the maximal bias $\alpha^{\mathbb{S}}$ by the PBE over the iterations. Figure \ref{fig:conv} shows this in the example of the Tetrahedron frame, where the PBE yields that $\alpha^{\mathbb{S}}=\boldsymbol{\frac{1}{\sqrt{3}}}$. We plot $\Vert \alpha^{(k)} - \alpha^{\mathbb{S}} \Vert$ as $k$ increases and find that the approximation is very slow, which emphasizes the superiority of the polytope approach in this setting.

\begin{figure}
\centering
\begin{subfigure}{.3\textwidth}
  \centering
  \includegraphics[width=1\linewidth]{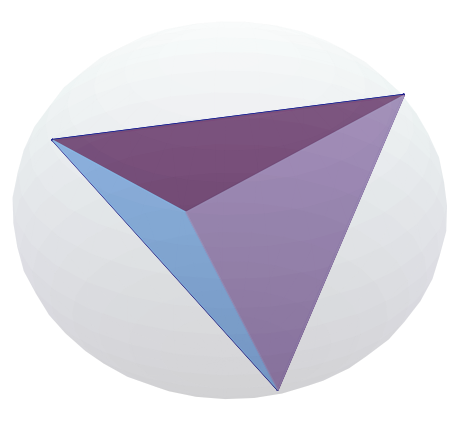}
  \label{fig:sub1}
\end{subfigure}%
\begin{subfigure}{.7\textwidth}
  \centering
  \includegraphics[width=0.9\linewidth]{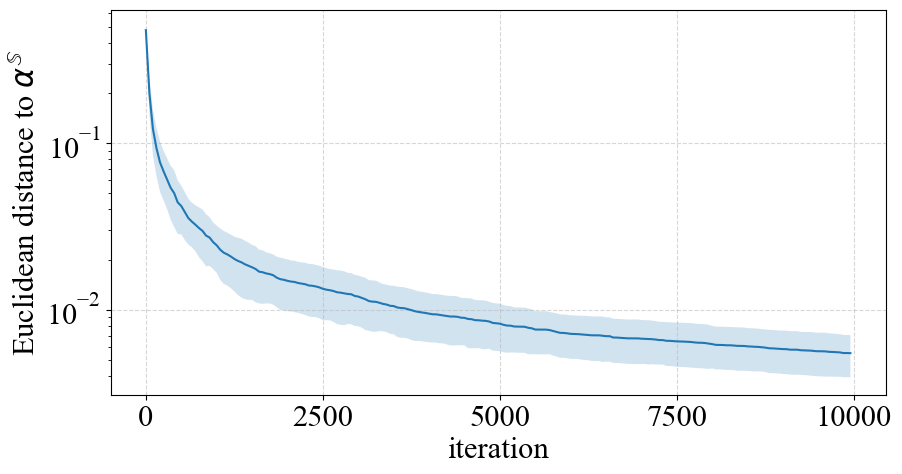}
  \label{fig:sub2}
\end{subfigure}
\caption{Left: The inscribing polytope for the Tetrahedron frame. Right: The Euclidean distance of $\alpha^{(k)}$ to the maximal bias of the Tetrahedron frame $\alpha^{\mathbb{S}}$ on a log y-scale over $10^5$ iterations. In cases where the inscribing polytope is a simplicial and regular polytope, the sampling-based bias estimation is sub-optimal, and the polytope approach already gives the maximal bias.}
\label{fig:conv}
\end{figure}


\subsubsection*{Remarks on the limitations}
Not surprisingly, both algorithms suffer from high dimensionality. For the sampling-based approach, the asymptotic behavior of the covering radius \eqref{eq:asym} indicates that it may become infeasible to reach a good approximation of $\alpha_K^\flat$ only by increasing the number of samples. Yet, with some experimenting on the factor, we can use expression \eqref{eq:asym} to effectively compensate for insufficient sampling in high dimensions. It has particularly high potential when injectivity is only required on specific data points of interest. In such a situation, the sampling set can be constructed in a custom data-driven way that respects the distribution of the data.

For the polytope bias estimation, the numerical computation of the convex hull to obtain the vertex-facet relations becomes infeasible in high dimensions. A possible remedy is to use dimensionality reduction and do the bias estimation in a lower dimensional space. The benefits of the method are that injectivity follows deterministically and that it comes with a lot of intuition an can be studied further using more advanced tools from convex geometry.\\

With this, we conclude the part of the paper that is concerned with the analysis of the injectivity of a ReLU layer. The last chapter is dedicated to the reconstruction of the input from the output of an injective ReLU layer.

\section{Duality and Reconstruction}\label{chap:duality}
If a ReLU layer $\Ta$ is injective, there is an inverse mapping that can infer any input from the output. This extends to the possibility of synthesizing new data that correspond to arbitrary coefficient vectors in the image of the ReLU layer, or studying the connection of single weights to the input by perturbing the corresponding output coefficient, thereby being able to interpret the output values.

In \cite{alharbi2024sat}, the authors propose to reconstruct the input from its saturated frame coefficients via a custom-modified version of the frame algorithm \cite{frames}. This idea can also be adapted to ReLU layers. Our main focus lies on another approach, where we propose to construct explicit perfect reconstruction formulas in the form of locally linear operators.

\subsection{ReLU-synthesis}
First, we recall the concept of a \textit{dual} frame, which is closely tied to two operators. The first one is the \emph{synthesis operator} which maps the frame coefficients back to the input space as
\begin{align*}
    D:\mathbb{R}^m &\rightarrow \mathbb{R}^n\\
    (c_i)_{i \in I}&\mapsto \sum_{i\in I} c_i \cdot \phi_i.
\end{align*}
The application of this operator is realized via the multiplication by the transpose of the analysis matrix from the left, $D=C^\top$. The second operator arises as the concatenation of analysis, followed by synthesis, also known as the \emph{frame operator},
\begin{align*}
    S:\mathbb{R}^n &\rightarrow \mathbb{R}^n\\
    x&\mapsto \sum_{i\in I} \fphi\cdot \phi_i.
\end{align*}
In matrix notation, $S=D C$.
The frame operator
is positive and self-adjoint, and if $\Phi$ is a frame, it is additionally invertible. Hence, one can write any $x\in \RR^n$ as
\begin{equation}\label{framedecomp}
    x = S^{-1}Sx = \sum_{i\in I} \fphi\cdot S^{-1} \phi_i.
\end{equation}
The collection $\tilde{\Phi}=\left(S^{-1} \phi_i\right)_{i\in I}$ is called the \emph{canonical dual frame} for $\Phi$.
Denoting the synthesis operator associated with $\tilde{\Phi}$ by $\tilde{D}$, then Equation \eqref{framedecomp} is equivalent to
\begin{equation}\label{eq:leftinv}
    \tilde{D}Cx = x.
\end{equation}
In other words, $\tilde{D}$ is a left-inverse of $C$ (given by the pseudo-inverse of $C$) \cite{casfin12}. This can be thought of as reconstructing $x$ from its frame coefficients using the canonical dual frame $\Tilde{\Phi}$. If $\Phi$ is redundant ($m>n$), there are infinitely many different possibilities of constructing a left-inverse of $C$. All of them can be interpreted as the synthesis operator of a (non-canonical) dual frame.
Using this machinery, we can define a reconstruction operator for $\Ta$ analogously to \eqref{eq:leftinv}. 
Note, however, that unless $\I \neq I$ for all $x\in K$, there is not \textit{one} reconstruction operator for all $x\in K$.


\begin{definition}
    The ReLU-synthesis operator associated with the collection $\Phi = (\phi_i)_{i\in I}\subset \RR^n$, the bias $\alpha \in \RR^m$, and the index set $J\subseteq I$ is defined by
    \begin{align}\label{eq:usso}
    \begin{split}
        D^{\alpha}_J:\mathbb{R}^m &\rightarrow \mathbb{R}^n\\
        (c_i)_{i \in I}&\mapsto \sum_{i\in J} \left(c_i + \alpha_i \right)\cdot \phi_i.
    \end{split}
    \end{align}
\end{definition}
Note that $\mathbb{R}^m$ is fixed as domain, the sum, however, runs over the index set $J$.
When using the ReLU-synthesis operator associated with a dual frame of $\Phi_J$ we obtain a reconstruction formula in the spirit of \eqref{eq:leftinv}.

\begin{theorem}[ReLU-dual I]\label{thm:leftinv}
    Let $\Phi$ be $\alpha$-rectifying on $K$ for $\alpha \in \RR^m$ and choose $x_0\in K$. Let $\widetilde{\Phi}_{\Io}=(\tilde{\phi}_i)_{i\in \Io}$ be any dual frame for $\Phi_{\Io}$, and $\tilde{D}^{\alpha}_{\Io}$ the associated ReLU-synthesis operator. Then for all $x\in K$ such that $\Io \subseteq \I$ it holds that
    \begin{equation}\label{eq:dual}
         \tilde{D}_{\Io}^{\alpha}\Ta x = x.
    \end{equation}
\end{theorem}

\begin{proof}
    For $x\in K$ with $\Io \subseteq \I$, the operator composition in \eqref{eq:dual} reduces to the usual frame decomposition with $\Phi_{\Io}$,
    \begin{align*}
         \tilde{D}_{\Io}^{\alpha}\Ta x &= \sum_{i\in \Io} 
        \left(\max(0,\langle x,\phi_i\rangle-\alpha_i)+\alpha_i \right)\cdot \tilde{\phi}_i\\
        &=\sum_{i\in \Io} \langle x,\phi_i\rangle\cdot \tilde{\phi}_i
        =x.
    \end{align*}
\end{proof}
The condition $\Io \subseteq \I$ for a reference vector $x_0\in K$ means that we may use the same left-inverse $\tilde{D}_{\Io}^{\alpha}$ for all input elements $x$ that share at least all active elements with $\Phi_{\Io}$.
By re-writing the condition on the level of index sets, we can alternatively choose a reference sub-space via an index set $J$ instead of fixing a reference vector $x_0$.

\begin{corollary}[ReLU-dual II]\label{thm:leftinv2}
    Let $\Phi$ be $\alpha$-rectifying on $K$ for $\alpha \in \RR^m$ and choose $J\subseteq I$ such that $\Phi_J$ is a frame. Let $\tilde{\Phi}_J$ be a dual frame for $\Phi_J$ and $\tilde{D}_{J}^{\alpha}$ the associated ReLU-synthesis operator. Then for all $x\in K$ with $J\subseteq I_x^{\alpha}$ it holds that
    $$
        \tilde{D}_{J}^\alpha\Ta  x = x.
    $$
\end{corollary}
The approach in the above corollary is particularly useful in the context of the bias estimation procedures described in Section \ref{sec:maxa}: Given a decomposition of $\Phi$ into sub-frames, either by all different most correlated bases or via the facets of $P_\Phi$, we can compute all associated left-inverses in advance and use them for reconstruction on demand. More precisely, for every $x\in K$ there is a sub-frame associated with a most correlated bases $J^*(x)$ such that $\tilde{D}_{J^*(x)}^\alpha$ is a left-inverse of $\Ta$. Similarly, if $\Phi$ is omnidirectional, then for every $x\in K$ there is a facet $F$ such that $\tilde{D}_{I_F}^\alpha$ is a left-inverse of $\Ta$. In both cases, there are only finitely many such sub-frames.
Summarizing, we have the following.

\begin{corollary}
    Let $\Phi$ be $\alpha$-rectifying on $K$. For any $x\in K$ there is $J\subseteq I$ such that $\tilde{D}_{J}^\alpha$ is a left-inverse of $\Ta$.
\end{corollary}

The numerical implementation of the ReLU-synthesis is straightforward when using canonical duals of the sub-frames $\Phi_J$. A detailed discussion and corresponding pseudo-code can be found in Appendix C.

\subsubsection*{Excursion: Reconstruction from PReLU layers} There are various modifications of the ReLU activation function, one of them being the parametrized ReLU, or PReLU, given by $\operatorname{PReLU}_\gamma = \max(\gamma s, s)$ with $0<\gamma\leq 1$  \cite{he2015prelu}. As this is an injective activation function, the associated PReLU layer with weights given by $\Phi$ and any bias $\alpha$ is injective if and only if $\Phi$ is a frame. In this case, for any $x\in K$ we obtain a left-inverse of the PReLU layer by
\begin{align}
\begin{split}
    \tilde{D}^{\alpha}_{\gamma}:\mathbb{R}^m &\rightarrow \mathbb{R}^n\\
    (c_i)_{i \in I}&\mapsto \sum_{i\in \I} \left(c_i + \alpha_i \right)\cdot \tilde{\phi}_i + \sum_{i\in I\setminus\I} \gamma^{-1}\left(c_i + \alpha_i \right)\cdot \tilde{\phi}_i,
\end{split}
\end{align}
where $\widetilde{\Phi}=(\tilde{\phi}_i)_{i\in I}$ is any dual frame for $\Phi$.

\subsection{The frame algorithm for ReLU layers}
The frame algorithm is an iterative scheme that constructs a sequence of vectors in $\RR^n$ from given frame coefficients $(\fphi)_{i\in I}$ that converges to the input $x$ exponentially fast \cite{frames}. This sequence $(y_k)_{k=0}^\infty$ is defined as $y_0=\mathbf{0}$ and
\begin{equation}\label{eq:framealg}
    y_{k+1} = y_{k} + \lambda \sum_{i\in I}\big(\langle x,\phi_i\rangle - \langle y_{k},\phi_i\rangle\big) \phi_i 
\end{equation}
for $k\geq 0$. Letting $A,B$ be the optimal frame bounds for $\Phi$ then for $0<\lambda<\tfrac{B}{2}$ we have for all $k\geq 0$ that $\Vert x-y_{k+1} \Vert \leq \kappa_\lambda \Vert x-y_{k} \Vert$, where $\kappa_\lambda = \max\{ |1-\lambda A|, |1-\lambda B| \}$. Note that the optimal value for the parameter $\lambda$ is $\tfrac{2}{A+B}$. In practice, the frame algorithm is a great tool to do reconstruction in situations, where computing the exact solution with the canonical dual (or any dual) frame becomes too expensive.

Clearly, the procedure can be directly applied for the reconstruction of $x$ from the output of a ReLU layer by reducing the sum in \eqref{eq:framealg} to run only over $\I$ (the active frame elements). To see that this is really what we want, note that for all $i\in \I$ we have $\langle x,\phi_i\rangle = \relu(\langle x,\phi_i\rangle-\alpha_i) + \alpha_i$. Therefore, the differences in the sum in \eqref{eq:framealg} over $\I$ are indeed taken between the values of the unbiased output of the ReLU layer and $\langle y_k,\phi_i\rangle$.
Hence, if our frame $\Phi$ is $\alpha$-rectifying on $K$ then for any $x\in K$ we
obtain a ReLU-reconstruction sequence $(y_k)_{k=0}^\infty$ that satisfies $\Vert x-y_{k+1} \Vert \leq \kappa_{x,\lambda} \Vert x-y_k \Vert$ for all $k\geq 0$, where $\kappa_{x,\lambda} = \max\{ |1-\lambda A_x|, |1-\lambda B_x| \}$, and $A_x,B_x$ are the optimal frame bounds for $\Phi_{\I}$.

However, we can do better than this. Following the idea in \cite{alharbi2024sat}, we can extend the frame algorithm by using the bias values $\alpha_i$ for all inactive frame elements as a proxy for the lost frame coefficients. This gives an algorithm that always outperforms the naive approach in the setting where we use the optimal parameter for the full frame.
\begin{proposition}[ReLU frame algorithm]\label{prop:framealg}
    Let $\Phi$ be $\alpha$-rectifying for $\alpha\in \RR^m$ on $K\subseteq \RR^n$. For any $x\in K$, let the ReLU-reconstruction sequence $(y_k)_{k=0}^\infty$ be given as $y_0=\mathbf{0}$ and
    \begin{equation}\label{eq:framealg2}
        y_{k+1} = y_{k} + \lambda \sum_{i\in \I}\big({\langle x,\phi_i\rangle} - \langle y_{k},\phi_i\rangle\big) \phi_i + \lambda_0\sum_{i\in I_{y_{k}}^\alpha\setminus\I}\big(\alpha_i- \langle y_{k},\phi_i\rangle\big) \phi_i
    \end{equation}
    for all $k\geq 0$.
    Let $\lambda = \tfrac{2}{A+B}$. If $\lambda_0=0$ then $\Vert x-y_{k+1} \Vert \leq \kappa_{\lambda} \Vert x-y_{k} \Vert$ for all $k\geq 0$, where $\kappa_\lambda = 1-A_x\tfrac{2}{A+B}$. If $\lambda_0 = \tfrac{2}{A+B}$ and for every $k\geq 0$ the optimal lower frame bound for $\Phi_{\I}$ is strictly less than the optimal frame bound for $\Phi_{\I \cup I_{y_k}^\alpha }$ then there is $0<\varepsilon_{x,y_k}<1$ such that
    \begin{equation}
        \Vert x-y_{k+1} \Vert \leq (1-\varepsilon_{x,y_k}) \kappa_{\lambda} \Vert x-y_{k} \Vert.
    \end{equation}
\end{proposition}
Since the proof is analogous to the one of Theorem 5.2 in \cite{alharbi2024sat}, we omit it here and refer to the appendix.

We note that using $\lambda = \lambda_0 = \tfrac{2}{A+B}$ as parameters in \eqref{eq:framealg2} is very natural as we may not always want to compute the optimal frame bounds for each activated sub-frame, but instead use a reasonably universal parameter that we only have to compute once. However, the design of the algorithm leaves it open to also use other parameters. A comprehensive analysis of which parameters work well is left as an open problem. Note further that the assumption on the frame bounds for $\Phi_{\I}$ and $\Phi_{\I \cup I_{y_k}^\alpha }$ is very mild. In fact, it is always fulfilled as long as $y_k$ does not lie in the span of one of the eigenvectors of the associated frame operator. In the worst case where this happens for all $k\geq 0$ then $\varepsilon_{x,y_k}=0$ and the extended frame algorithm is as fast as the naive one.

\subsection{Stability of the reconstruction}\label{sec:stab}
In this section, we revisit a result by \cite{puth22} on the lower Lipschitz bound of a ReLU layer and translate it into the language of frame theory as done in \cite{bandeira2014savingphase} and \cite{alharbi2024sat}. As a small extension, we present a result on the local lower Lipschitz stability. A general revision of the Lipschitz stability analysis of ReLU layers is left for future work.
\begin{definition}
    A frame $\Phi$ allows $\kappa$-stable $\alpha$-rectification on $K$ if there is $\kappa>0$ s.t.
    \begin{equation}\label{eq:stabglob}
        \| y-z \|^2 \leq \kappa\cdot \| \Ta y - \Ta z \|^2
    \end{equation}
    holds for all $y,z\in K$.
\end{definition}
First of all, we note that if $\Phi$ is $\alpha$-rectifying on $K$ then a frame-type inequality as \eqref{eq:framedef} holds for $\Ta$. That is, there are constants $0< A_\alpha\leq B_\alpha<\infty$ such that
\begin{equation}\label{reluframe}
    A_\alpha\cdot \|x\|^2\leq \|\Ta x\|^2
    \leq B_\alpha\cdot\| x\|^2
\end{equation}
for all $x\in K$. It is easy to see that the largest possibility of choosing the lower bound $A_\alpha$ in \eqref{reluframe} is the smallest lower frame bound among all possible active sub-frames $\Phi_{\I}$ with $x\in K$. Analogously, the smallest possibility for the upper bound $B_\alpha$ coincides with the largest of all upper frame bounds.
In \cite{puth22} it was shown that any $\alpha$-rectifying frame allows $(2m A_\alpha\inv)$-stable $\mathbf{0}$-rectification on $\RR^n$ but not $(A_\alpha\inv)$-stable $\mathbf{0}$-rectification. It remains an open problem whether this statement extends to the case of non-zero biases and if the factor $m$ can be replaced by a constant.

In general,
active sub-frames are naturally prone to have a bad lower frame bound, such that $A_\alpha\inv$ may become very large, and the problem becomes globally ill-conditioned. To get a better understanding of how stable the reconstruction process is for smaller portions of the data, we shall investigate the lower Lipschitz property locally.

\begin{definition}\label{locstab}
    For $x_0\in K$, a frame $\Phi$ allows $\kappa_{x_0}$-stable $\alpha$-rectification near $x_0$ if there is $\varepsilon > 0$ and $\kappa_{x_0}=\kappa(x_0)>0$ such that
    \begin{equation}\label{eq:locstab}
        \| y-z \|^2 \leq \kappa_{x_0} \cdot \| \Ta y - \Ta z \|^2
    \end{equation}
    holds for all $y,z \in \mathring{B}_{\varepsilon}(x_0).$
\end{definition}
Since locally, a ReLU layer is a linear map we might hope to use $A_\alpha\inv$ as a lower bound. In general, however, we can only guarantee that an $\alpha$-rectifying frame allows $A_\alpha\inv$-stable $\beta$-rectification for $\beta<\alpha$.
\begin{proposition}
    Let $\Phi$ be $\alpha$-rectifying on $K$. For $x_0\in K$ let $J=J(x_0)\subseteq \Io$ be such that $\Phi_{J}$ is a frame with lower frame bound $A_{J}$. Then $\Phi$ allows $A_{J}\inv$-stable $\beta$-rectification near $x_0$ for $\beta<\alpha$.
\end{proposition}
\begin{proof}
    Let $x_0\in K$. There is $J=J(x_0)\subseteq I_{x_0}^{\alpha}$ such that $\Phi_J$ is a frame with frame bounds $0 < A_{J} \leq B_{J} < \infty$.
    Analog to the global case, the bi-Lipschitz condition
    \begin{equation}\label{eq:tb}
        A_{J} \cdot \| y-z \|^2 \leq \| \Ta y - \Ta z \|^2 \leq B_{J} \cdot \| y-z \|^2
    \end{equation}
    holds for all $y,z\in K$ with $I_y^{\alpha}, I_z^{\alpha} \supseteq J $. Let $\beta \in \RR^m$ such that
    $\langle x_0, \phi_i \rangle \geq \alpha_i>\beta_i$ for all $i\in J$. Hence, there is $\varepsilon > 0$ sufficiently small such that for $y,z\in \mathring{B}_{\varepsilon}(x_0)$ we have $I_y^{\beta}, I_z^{\beta} \supseteq J $. Since \eqref{eq:tb} still holds for $C_\beta$, we get that $\Phi$ allows $A_{J}\inv$-stable $\beta$-rectification near $x_0$.
\end{proof}

\begin{corollary}
    Let $\Phi$ be $\alpha$-rectifying on $K$, and $x_0\in K$ such that $\langle x_0, \phi_i \rangle > \alpha_i$ for all $i\in J$. Then $\Phi$ allows $A_{J}\inv$-stable $\alpha$-rectification near $x_0$.
\end{corollary}

\subsection{Image of ReLU layers}
In the context of our work, it is specifically interesting to know what the image of a ReLU layers looks like. The impact is two-fold.
\begin{itemize}
    \item[(1)] To study the injectivity of a ReLU layer that applies to the output of a previous one it is crucial to know how its image looks like.
    \item[(2)] Reconstructing samples from the image of a ReLU layer yields a valuable method to generate new consistent data and understand the effect of their ReLU layer.
\end{itemize}
We give a partial answer to this question in the following. For bounded $K$ we can use the upper bound in \eqref{reluframe} to find the smallest closed non-negative ball in $\RR^m$ that contains $C_\alpha\left(K\right)$.
\begin{lemma}\label{lem:image}
    Let $\Phi$ be $\alpha$-rectifying on $K$ bounded with $M=\sup_{x\in K}\Vert x \Vert$. Letting $B_\alpha$ denote the largest optimal upper frame bound among all sub-frames $\Phi_{\I}$ with $x\in K$, then
    \begin{equation}\label{eq:estim}
        C_\alpha\left(K\right) \subseteq \mathbb{B}_{\sqrt{B_\alpha}M}^+, 
    \end{equation}
    where $\sqrt{B_\alpha}M$ is the minimal radius.
\end{lemma}

\begin{proof}
    By \eqref{reluframe}, for $x\in K$ we have that $\|\Ta x\|^2 \leq B_\alpha \|x\|^2 \leq B_\alpha M^2$. The application of the ReLU function then corresponds to the projection onto the non-negative part of $\mathbb{B}_{\sqrt{B_\alpha}M}$, i.e., $\mathbb{B}_{\sqrt{B_\alpha}M}^+$. Since all estimations are sharp, the claim follows.
\end{proof}
As already mentioned, understanding the images of specific sets under ReLU layers is crucial to understanding how data is processed and passed on to the next layer. The above lemma is just a small step towards revealing this knowledge which can be used for unraveling certain behaviors of neural networks, and further, enhancing their interpretability and transparency.

\section{Conclusion}
This manuscript studies the injectivity of ReLU layers and the exact recovery of input vectors from their output using frame theory as a tool. Among many basic properties and insights about ReLU layers on bounded domains, the main theoretical contribution is three different characterizations of the injectivity of ReLU layers that together provide a complete picture of its injectivity behavior as a non-linear deterministic operator. A significant portion of the research focuses on the computation of a maximal bias for a given frame $\Phi$ and a domain $K$, such that the associated ReLU layer is injective on $K$. This characterization is particularly interesting as it allows us to apply the theoretical results in practical applications. We discuss two different methods to approach this, both of which have distinct advantages and disadvantages, and provide algorithmic solutions to compute approximations of a maximal bias in practice.
The second part of this work is devoted to the derivation of reconstruction formulas for injective ReLU layers, based on the concept of duality in frame theory. A brief local stability analysis of the reconstruction operator completes the discussion.

In summary, this paper provides a methodology on how to study the channeling of information in ReLU layers with biases and given input data, made possible by using frame theory as a tool. The results are designed in a general, yet, accessible way such that they may stimulate further theoretical research, but are also directly applicable in practice. While we are pleased to contribute to advancing the understanding of these fundamental and ubiquitous building blocks of neural networks, many critical aspects remain to be explored. A central question in this context is how information propagates through ReLU \textit{networks}, so how can we rigorously characterize the injectivity of the composition of ReLU layers.

\section*{Acknowledgement}
D. Haider is recipient of a DOC Fellowship of the Austrian Academy of Sciences at the Acoustics Research Institute (A 26355). The work of P. Balazs was supported by the FWF projects LoFT (P 34624) and NoMASP (P 34922). 
The authors would particularly like to thank Daniel Freeman for his valuable input during very enjoyable discussions and Hannah Eckert for her work on the numerical experiments.


\bibliographystyle{siamplain}
\bibliography{references}

\clearpage
\section*{Appendix}

\subsection*{A - Remarks on Admissibility and Directed Spanning Sets}
With this short comment, we aim to complete the circle between three perspectives to characterize the injectivity of a ReLU layer on $\RR^n$.
In the work by Bruna et al. \cite{bruna14} the injectivity of a ReLU layer, or \textit{half-rectification operator} was linked to an admissibility condition of a $J\subseteq I$. There, $J$ is called admissible for $\Phi$ and $\alpha$ if
\begin{equation*}
    \bigcap_{i\in J}\{x\in\mathbb{R}^n:\langle x,\phi_i\rangle > \alpha_i\}\cap \bigcap_{i\notin J}\{x\in\mathbb{R}^n:\langle x,\phi_i\rangle < \alpha_i\}\neq \emptyset.
\end{equation*}
Puthawala et al. in \cite{puth22} already pointed out that Proposition $2.2$ in \cite{bruna14} is not exactly equivalent to the injectivity of a ReLU layer.
However, with a slight modification to
\begin{equation*}
    \bigcap_{i\in J}\{x\in\mathbb{R}^n:\langle x,\phi_i\rangle \geq \alpha_i\}\cap \bigcap_{i\notin J}\{x\in\mathbb{R}^n:\langle x,\phi_i\rangle < \alpha_i\}\neq \emptyset,
\end{equation*}
this stands in direct relation to the index sets $\I$, as introduced in Definition \ref{alpharect} in the present manuscript. Indeed, with this modified definition, $J$ is $\alpha$-admissible for $\Phi$ if and only if there is $x\in\mathbb{R}^n$ such that $J=\I$. 
As a consequence, we have that the equivalence of $(i)$ and $(ii)$ in Corollary \ref{thm:bigthm} here with $K=\RR^n$, Proposition $2.2$ in \cite{bruna14} with the modified admissibility condition, and Theorem $2$ in \cite{puth22} are equivalent.

\subsection*{B - Omnidirectionality}
We prove the statement about omnidirectionality mentioned in Approach B. of Section \ref{sec:maxa}: By adding a single vector, any non-omnidirectional frame can be made omnidirectional. Furthermore, we recall how omnidirectionality can be checked numerically as in \cite{behr18}.
\begin{lemma}
Let $\Phi$ be a non-omnidirectional frame, then $$\Phi'=\left( \Phi , -\frac{\sum_{i\in I} \phi_i}{\|\sum_{i\in I} \phi_i\|} \right)$$ is omnidirectional.
\end{lemma}
\begin{proof}
At first, note that if $\sum_{i\in I} \phi_i=0$, then $\Phi$ is already omnidirectional. Let $c_1,...,c_m=\frac{1}{1+\|\sum_{i\in I} \phi_i\|}$ and $c_{m+1} = \frac{\|\sum_{i\in I} \phi_i\|}{1+\|\sum_{i\in I} \phi_i\|}$ and $\phi_{m+1} = -\frac{\sum_{i\in I} \phi_i}{\|\sum_{i\in I} \phi_i\|}$, then
\begin{equation}\label{eq:omni2}
    \sum_{i=1}^{m+1} c_i\cdot \phi_i = 0 .
\end{equation}
Note that $c_i > 0$ for all $i=1,\dots ,m+1$ and $\sum_{i=1}^{m+1} c_i = 1$. Hence, in \eqref{eq:omni2} we wrote $0$ as a convex combination of \emph{all} elements of $\Phi'$. This implies that $0\in \mathring{P}_{\Phi'}$, hence $\Phi'$ is omnidirectional.
\end{proof}
Regarding the verification of omnidirectionality, let $D$ be the synthesis matrix associated with $\Phi$, i.e., it consists of the column vectors $\phi_i$ for $i\in I$. Then, verifying omnidirectionality is equivalent to the existence of a solution for the convex optimization problem $\min\ \Vert Dc \Vert \ \text{subject to}\  c > 0$.

\begin{figure}[t]
    \begin{subfigure}[t]{0.49\textwidth}
        \centering
        \includegraphics[width=0.5\textwidth]{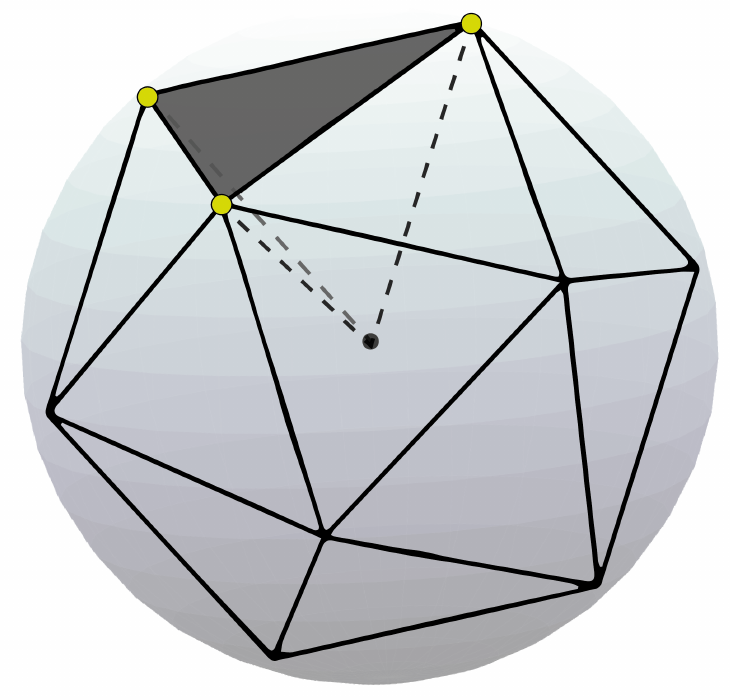}
    \end{subfigure}%
    \hfill
    \begin{subfigure}[t]{0.49\textwidth}
        \centering
        \includegraphics[width=0.5\textwidth]{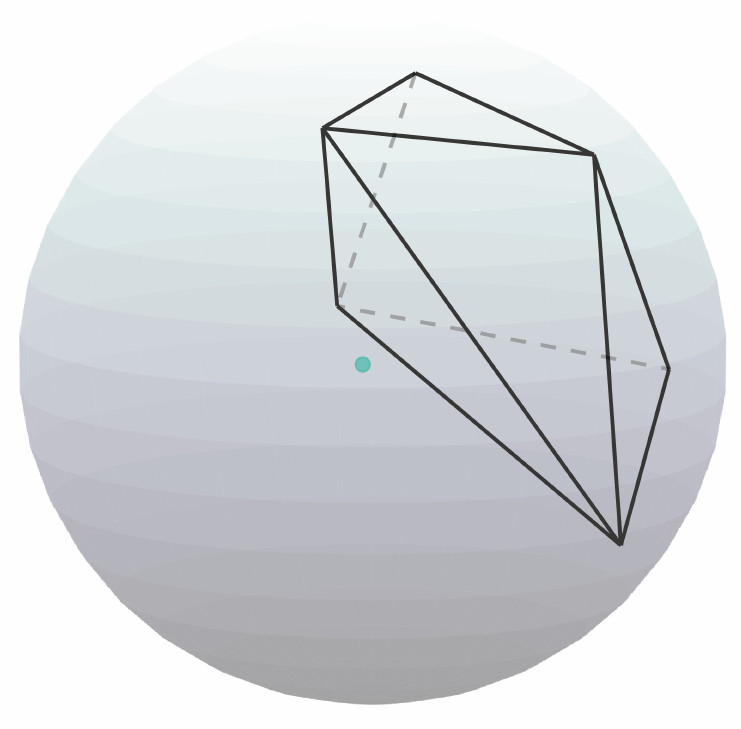}
    \end{subfigure}%

    \caption{Left: The frame is omnidirectional. Right: The frame is not omnidirectional.}
    \label{fig:omni}
\end{figure}

\subsection*{C - Algorithms}
We discuss the implementation of the presented algorithmic approaches and provide detailed pseudo-code. Our Python implementations can be found under \href{https://github.com/danedane-haider/Alpha-rectifying-frames}{https://github.com/danedane-haider/Alpha-rectifying-frames}.\\

\textbf{C1. Sampling-based bias estimation.}
Given a frame $\Phi$, a data domain $K$, and a sequence of samples $X_N\subset K$, Algorithm \ref{alg:mcbe} demonstrates the sampling-based bias estimation presented in Theorem \ref{thm:samp}. The samples $x_k\in X_N$ can be chosen to be random samples, e.g., $x_k \sim \mathcal{U}(K)$ for suitable $K$. Assuming the frame to be full-spark, then $J^*(x_k)$ consists of the indices for the largest $n$ frame coefficients $ (\langle x_k,\phi_i\rangle)_{i\in I}$. 

\begin{algorithm}[hbt!]
\begin{tcolorbox}[colback=lightgray!25!white, colframe=lightgray!25!white, width=\textwidth, sharp corners, boxsep=5pt, left=2mm, right=2mm, top=0mm, bottom=0mm]
\caption{Sampling-based approach for approximating $\alpha_K^{\flat}$}
\label{alg:mcbe}
\begin{algorithmic}
    \State \textbf{Input:} $\Phi, X_N, K$
    \State initialize $(\alpha^{(0)}) = \alpha_{\Phi}^\Delta$, $k=0$
    \For{$z$ in $X_N$}
        \State compute $(\langle z,\phi_i\rangle)_{i\in I}$
        \State get $J^*(z)$
        \State update $(\alpha^{(k+1)})_{i} \leftarrow \min \{ \langle z,\phi_i \rangle, (\alpha^{(k)})_{i} \}$ for all $i\in J^*(z)$
        \State $k=k+1$
    \EndFor
\end{algorithmic}
\end{tcolorbox}
\vspace{-0.7cm}
\end{algorithm}
\noindent
The for-loop can be replaced with a while-loop conditioned on $\Vert \alpha^{(k+steps)}- \alpha^{(k)}\Vert>\varepsilon>0$, where the parameter $steps$ determines how many updates should be done before checking the condition. This is very useful to avoid early stopping.\\

\textbf{C2. Polytope bias estimation.}
Given a frame $\Phi$ and a bounded domain $K$. The vertex-facet relations for $P_\Phi$ encoded in $I_{F_j}$ can be computed with convex hull algorithms, e.g., using the property \texttt{simplices} from \texttt{scipy.spatial.ConvexHull} in Python.
In the following, we demonstrate how to compute the biases from Proposition \ref{prop:pbeall}.\\

\noindent
$(i)$ $K=\partial P_\Phi:$ Recall that $\alpha_{\Phi}^\Delta$ is given by
\begin{equation}
    \left(\alpha_{\Phi}^\Delta\right)_i = \min_{
    \substack{\ell \in I_{F_j}\\j:\phi_i\in F_j}
    } \langle \phi_{\ell},\phi_i \rangle.
\end{equation}

\begin{algorithm}[hbt!]
\begin{tcolorbox}[colback=lightgray!25!white, colframe=lightgray!25!white, width=\textwidth, sharp corners, boxsep=5pt, left=2mm, right=2mm, top=0mm, bottom=0mm]
\caption{PBE for $\partial P_\Phi$}
\label{alg:pbe1}
\begin{algorithmic}
\State \textbf{Input:} $\Phi$
\State compute $I_{F_j}$ for all facets
\For{$j=1, \dots, \#$facets}
    \State $\beta_j = \min_{k<\ell\in I_{F_j} } \langle \phi_k,\phi_\ell\rangle$
    \EndFor
    \For{$i=1,\dots,m$}
        \State $\left(\alpha_{\Phi}^\Delta\right)_i = \min\limits_{j : i\in I_{F_j}} \beta_j$
    \EndFor
\end{algorithmic}
\end{tcolorbox}
\vspace{-0.7cm}
\end{algorithm}
\noindent
$(ii)$ $K=\mathbb{S}:$ Recall that $\alpha_{\mathbb{S}}^\Delta$ is given by
\begin{equation}\label{eq:as}
    \left(\alpha_{\mathbb{S}}^\Delta\right)_i = \min\{\min_{
    \substack{y \in F_j^\mathbb{S}\\j:\phi_i\in F_j}
    } \langle y , \phi_i \rangle, \left(\alpha_{\Phi}^\Delta\right)_i\}.
\end{equation} 
First, we show that for fixed $i\in I_{F_j}$,
\begin{equation}\label{eq:conv}
    \min_{
    \substack{y \in F_j^\mathbb{S}\\j:\phi_i\in F_j}
    } \langle y , \phi_i \rangle
\end{equation}
can be computed using convex linear programs.
Letting $C_{I_{F_j}}, D_{I_{F_j}}$ denote the analysis and synthesis operator associated with $\Phi_{I_{F_j}}$, respectively. We can write any $x\in F_j^{\mathbb{S}}$ as $x = \sum_{\ell\in I_{F_j}} c_\ell \phi_\ell=D_{I_{F_j}}c$ for some vector $c\geq0$. So for any $i\in I_{F_j}$ the solution of \eqref{eq:conv} is found by solving the linear program
\begin{align}\label{eq:opt}
\begin{split}
    \min_{j:\phi_i\in F_j}\ \left(C_{I_{F_j}} D_{I_{F_j}}c\right)_i \\
    \text{subject to}\  c &\geq 0\\
    \|D_{I_{F_j}} c\|_2 &= 1.
\end{split}
\end{align}
If $\left(\alpha_{\Phi}^\Delta\right)_i<0$, then the above minimum is negative since $\left(\alpha_{\mathbb{S}}^\Delta\right)_i\leq\left(\alpha_{\Phi}^\Delta\right)_i<0$. Therefore, we can replace $\|D_{I_{F_j}} c\|_2=1$ by $\|D_{I_{F_j}} c\|_2\leq 1$ making the problem convex.\\

\begin{algorithm}[hbt!]
\begin{tcolorbox}[colback=lightgray!25!white, colframe=lightgray!25!white, width=\textwidth, sharp corners, boxsep=5pt, left=2mm, right=2mm, top=0mm, bottom=0mm]
\caption{PBE for $\mathbb{S}$}
\label{alg:pbe2}
\begin{algorithmic}
\State \textbf{Input:} $\Phi$
\State compute $\alpha_\Phi^\Delta$
    \For{$i=1,\dots,m$}
        \If{$\left(\alpha_{\Phi}^\Delta\right)_i \geq 0$} 
                    \State $\left(\alpha_{\mathbb{S}}^\Delta\right)_i \gets \left(\alpha_{\Phi}^\Delta\right)_i$
                \Else
                    \State $\left(\alpha_{\mathbb{S}}^\Delta\right)_i \gets $ solution of \eqref{eq:opt}
        \EndIf
    \EndFor
\end{algorithmic}
\end{tcolorbox}
\vspace{-0.7cm}
\end{algorithm}

\noindent
$(iii)$ $K=\mathbb{D}_{r,s}:$ Let $0\leq s<r$  and recall that $\alpha_{\mathbb{B}}^\Delta$ is given by
\begin{equation}\label{eq:B}
    \left(\alpha_{\mathbb{B}}^\Delta\right)_i = \min\{s,\left(\alpha_{\mathbb{S}}^\Delta\right)_i\}.
\end{equation}
One gets the general case by scaling with $r^{-1}$. The case $s=0$ yields a bias estimation for $\mathbb{B}_r$. Since \eqref{eq:B} depends on $\alpha_{\mathbb{S}}^\Delta$ in a trivial way, we omit the algorithm.\\


\noindent
$(iv)$ $K=\mathbb{B}^+_r:$ Let $e\in \RR^m$ be arbitrary and recall that $\alpha^{\mathbb{B}^+}$ is given by
\begin{equation}
\left(\alpha^\Delta_{\mathbb{B}^+}\right)_i=
\begin{cases}
    \left(\alpha^\Delta_\mathbb{B}\right)_i &\text{ for } i \in I^+\\
    s &\text{ else,} 
\end{cases}
\end{equation}
where $s$ is arbitrary.
The crux here is to compute the index set $ I^+ = \bigcup_{j\in J^+} I_{F_j}$ defined via $J^+=\{j\in I : F_j\cap \RR^n_+\neq \emptyset\}$. One way to verify that $F_j\cap \RR^n_+\neq \emptyset$ is to check the feasibility of the convex optimization problem
\begin{align}\label{eq:opt++}
   \min\ \|D_{I_{F_j}}c\|_2 \nonumber \\
   \text{subject to}\ c&\geq 0\\
   \sum_i c_i &= 1.\nonumber
\end{align}
If \eqref{eq:opt++} has a solution for the facet $F_j$, then there is $c\in \RR^n_+$ that can be written as a convex linear combination of the vertices of $F_j$, hence, $F_j\cap \RR^n_+\neq \emptyset$.\\

\textbf{C3. ReLU-duals and reconstruction.}
For $J\subseteq I$ we denote by $C_J$ the analysis operator for the collection $\Phi_J$. This corresponds to the $\vert J \vert \times n$ matrix $C_J$ consisting of the row vectors $\phi_i$ for $i\in J$. Algorithm \ref{alg:reludual2} describes how to build the synthesis matrices for doing ReLU-synthesis for an index set $J$ from Corollary \ref{thm:leftinv2}.

\begin{algorithm}[hbt!]
\begin{tcolorbox}[colback=lightgray!25!white, colframe=lightgray!25!white, width=\textwidth, sharp corners, boxsep=5pt, left=2mm, right=2mm, top=0mm, bottom=0mm]
\caption{Construction of the matrix for ReLU-synthesis}
\label{alg:reludual2}
\begin{algorithmic}
    \State \textbf{Input:} $\Phi, J=\{i_1,\dots, i_{|J|}\}\subseteq I$ such that $\Psi = \Phi_J$ is a frame
    \State $S^{-1}_J\gets\left( (C_J)^\top C_J \right)^{-1}$
    \State $\Tilde{D}_J \gets
        \begin{pmatrix}
            \vert & \vert &  & \vert \\ \\
            S^{-1}_J \psi_{i_1} & S^{-1}_J \psi_{i_2} & \cdots & S^{-1}_J \psi_{i_{\vert J \vert}}\\ \\
            \vert & \vert &  & \vert
        \end{pmatrix}$
    \State 
\end{algorithmic}
\end{tcolorbox}
\vspace{-0.7cm}
\end{algorithm}
\noindent
Note that if we insert columns of zeros in $\Tilde{D}_J$ for all coordinates that have not been activated, then the resulting matrix would be the synthesis operator associated with a non-canonical dual frame for $\Phi$ in the classical sense. Applying it to an output vector that is restricted to only the coordinates in $J$, as presented here, computes the corresponding input by avoiding unnecessary multiplications with zeros. Finally, we want to perform the actual reconstruction: Given $z = \Ta x_0$ for some $x_0 \in K$, we can read off $\Io$ from $z$ directly (under the assumption that $\langle x_0, \phi_i \rangle \neq \alpha_i$ for all $i\in I$). Assuming that we have a suitable list of index sets of sub-frames (e.g., all most correlated bases of all facet sub-frames), finding those that are contained in $\Io$ is easy. Choose one, say $J$. The choice $J=\Io$ is valid too. Then the matrix $\Tilde{D}_J$ provides reconstruction as follows.

\begin{algorithm}[hbt!]
\begin{tcolorbox}[colback=lightgray!25!white, colframe=lightgray!25!white, width=\textwidth, sharp corners, boxsep=5pt, left=2mm, right=2mm, top=0mm, bottom=0mm]
\caption{Applying the ReLU-synthesis: Reconstruction of $x_0$ from $z = \Ta x_0$}
\label{alg:recon}
\begin{algorithmic}
    \State \textbf{Input:} $\Phi, \alpha, z$
    \State find $J\subseteq \Io$ such that $\Phi_J$ is a frame
    \State $ z \gets z+\alpha $ (unbias)
    \State restrict to $J$ via $\zeta \gets (z_i)_{i\in J} $
    \State reconstruct $\Tilde{D}_J\zeta=x_0$ (see Algorithm \ref{alg:reludual2})
\end{algorithmic}
\end{tcolorbox}
\vspace{-0.7cm}
\end{algorithm}

\textbf{C4. ReLU frame algorithm.}
For the sake of completeness, we give a proof of Proposition \ref{prop:framealg} that follows the lines of the one of Theorem 5.2 in \cite{alharbi2024sat}.


\begin{proof}[Proof of Proposition \ref{prop:framealg}]
    At first note that applying the (classical) frame algorithm \eqref{eq:framealg} with $\Phi_{\I}$ and $\lambda = \tfrac{2}{A+B}$ gives the constant $\kappa_\lambda = 1-A_x\tfrac{2}{A+B}$ since $\tfrac{2}{A+B}<\tfrac{2}{A_x+B_x}$. Now, for every $i\in I_{y_{k}}^\alpha\setminus\I$ we define $$\gamma_i = \frac{\alpha_i-\langle y_k,\phi_i\rangle}{\langle x-y_k,\phi_i\rangle}$$ and note that $0\leq\gamma_i<1$. It is easy to see that the extended frame algorithm \eqref{eq:framealg2} constructs the same sequence of vectors $(y_k)_{k=0}^\infty$ as applying the (classical) frame algorithm \eqref{eq:framealg} using the frame $\Phi_{\gamma}=(\phi_i)_{i\in \I}\cup (\gamma^{\nicefrac{1}{2}}\phi_i)_{i\in I_{y_{k}}^\alpha}$. Let $A'$ denote the optimal lower frame bound for $\Phi_{\I \cup I_{y_k}^\alpha }$ and $A''$ the one for $\Phi_{\gamma}$. Then, by assumption that $A_x<A'$, together with $\gamma_i<1$, we have $A_x<A'<A''\leq A$. Applying the convergence result for the (classical) frame algorithm for $\Phi_\gamma$ gives
    \begin{equation}
        \Vert x-y_{k+1} \Vert \leq \left(1-A''\tfrac{2}{A+B}\right)\Vert x-y_{k} \Vert \leq (1-\varepsilon_{x,y_k})\underbrace{\left(1-A_x\tfrac{2}{A+B}\right)}_{\kappa_{\lambda}} \Vert x-y_{k} \Vert,
    \end{equation}
    where $0<\varepsilon_{x,y_k}<1$.
\end{proof}

\end{document}